\documentclass[12pt]{article}
\usepackage[margin=1in]{geometry}

\usepackage{amsmath,amssymb,amsthm}
\usepackage{pgfplots}
\usepackage{comment}
\usepackage[textsize=footnotesize,color=green!40]{todonotes}
\usepackage[utf8]{inputenc}
\usepackage{amsfonts}
\usepackage{amssymb}
\usepackage{mdframed}
\usepackage{framed}
\usepackage{hyperref}
\usepackage{mathtools}
\usepackage{authblk}
\usepackage{fancyhdr}
\usepackage{float}

\usepackage[noend]{algorithmic}
\usepackage[algoruled,vlined,nofillcomment,algo2e]{algorithm2e}
\usepackage{bm}

\usepackage{tikz}
\usetikzlibrary{calc,backgrounds}
\usetikzlibrary{arrows,calc,decorations.pathreplacing,positioning,shapes,shapes.geometric,shapes.callouts}
\usetikzlibrary{shapes,intersections,external}
\usepackage{pgfplots}

\newcommand{\dom}[1]{\mathbb{#1}}
\newcommand{\mech}[1]{\mathcal{#1}}
\newcommand{\tup}[1]{\vec{#1}}
\newcommand{\rv}[1]{\mathsf{#1}}
\newcommand{\rvset}[1]{\bm{\rv{#1}}}

\renewcommand{\Pr}{\mathbb{P}}
\newcommand{\Ex}{\mathbb{E}}
\newcommand{\Var}{\mathbb{V}}
\newcommand{\EE}{\mathbb{E}}
\newcommand{\PP}{\mathbb{P}}
\newcommand{\TV}{\mathfrak{T}}

\renewcommand{\div}{\mathfrak{D}}
\newcommand{\mse}{\mathrm{MSE}}
\newcommand{\bernoulli}{\texttt{Ber}}
\newcommand{\uniform}{\texttt{Unif}}
\newcommand{\debias}{\texttt{DeBias}}
\newcommand{\binomial}{\texttt{Bin}}

\newcommand{\Normal}{\texttt{N}}
\newcommand{\Laplace}{\texttt{Lap}}

\newcommand{\one}{\mathbb{I}}
\newcommand{\view}{\mathrm{View}}

\newcommand{\N}{\mathbb{N}}
\newcommand{\R}{\mathbb{R}}
\newcommand{\I}{I}

\newcommand{\D}{\mathcal{D}}

\newcommand{\image}{\mathsf{Im}}

\usepackage[textsize=footnotesize,color=green!40]{todonotes}

\newcommand{\remove}[1]{}

\DeclarePairedDelimiter{\ceil}{\lceil}{\rceil}

\usepackage{amsthm}

\newtheorem{theorem}{Theorem}[section]
\newtheorem{lemma}{Lemma}[section]

\newtheorem{corollary}{Corollary}[theorem]

 \date{}

\title{The Privacy Blanket of the Shuffle Model}
\author{Borja Balle~~~~~James Bell\thanks{The Alan Turing Institute. {\tt jbell@posteo.net}. Some of this work was done at Cambridge University, and party supported by the UK Government’s Defence \& Security Programme in support of the Alan Turing Institute.}~~~~~Adria Gascon\thanks{The Alan Turing Institute and Warwick University. {\tt agascon@turing.ac.uk}. Work supported by The Alan Turing Institute under the EPSRC
grant EP/N510129/1, and the UK Government’s Defence \& Security Programme in support of the Alan Turing Institute.}~~~~~Kobbi Nissim\thanks{Dept.\ of Computer Science, Georgetown University. {\tt kobbi.nissim@georgetown.edu}. Work supported by NSF grant no.~1565387, 
TWC: Large: Collaborative: Computing Over Distributed Sensitive Data. Work partly done while K.~N.\ was visiting the Alan Turing Institute.}}

\begin{document}

\maketitle

\begin{abstract}
This work studies differential privacy in the context of the recently proposed \emph{shuffle model}.
Unlike in the local model, where the server collecting privatized data from users can track back an input to a specific user, in the shuffle model users submit their privatized inputs to a server anonymously.
This setup yields a trust model which sits in between the classical curator and local models for differential privacy.
The shuffle model is the core idea in the Encode, Shuffle, Analyze (ESA) model introduced by Bittau et al.~(SOPS 2017).
Recent work by Cheu et al.~(EUROCRYPT 2019) analyzes the differential privacy properties of the shuffle model and shows that in some cases shuffled protocols provide strictly better accuracy than local protocols.
Additionally, Erlingsson et al.~(SODA 2019) provide a privacy amplification bound quantifying the level of curator differential privacy achieved by the shuffle model in terms of the local differential privacy of the randomizer used by each user.

In this context, we make three contributions.
First, we provide an optimal single message protocol for summation of real numbers in the shuffle model.
Our protocol is very simple and has better accuracy and communication than the protocols for this same problem proposed by Cheu et al.
Optimality of this protocol follows from our second contribution, a new lower bound for the accuracy of private protocols for summation of real numbers in the shuffle model.
The third contribution is a new amplification bound for analyzing the privacy of protocols in the shuffle model in terms of the privacy provided by the corresponding local randomizer.
Our amplification bound generalizes the results by Erlingsson et al.\ to a wider range of parameters, and provides a whole family of methods to analyze privacy amplification in the shuffle model.
\end{abstract}

\section{Introduction}
\label{sec:introduction}
Most of the research in differential privacy focuses on one of two extreme models of distribution. In the  curator model, a {\em trusted} data collector assembles users' sensitive personal information and analyses it while injecting random noise strategically designed to provide both differential privacy and data utility. In the local model, each user $i$ with input $x_i$ applies a local randomizer $\mech{R}$ on her data to obtain a message $y_i$, which is then submitted to an {\em untrusted} analyzer. Crucially, the randomizer $\mech{R}$ guarantees differential privacy independently of the analyzer and the other users, even if they collude. Separation results between the local and curator models are well-known since the early research in differential privacy: certain learning tasks that can be performed in the curator model cannot be performed in the local model~\cite{KLNRS08} and, furthermore, for those tasks that can be performed in the local model there are provable large gaps in accuracy when compared with the curator model. An important example is the summation of binary or (bounded) real-valued inputs among $n$ users, which can be performed with $O(1)$ noise in the curator model~\cite{DMNS06} whereas in the local model the noise level is $\Omega(\sqrt{n})$~\cite{BNO08,ChanSS12}. Nevertheless, the local model has been the model of choice for recent implementations of differentially private protocols by Google~\cite{EPK14}, Apple~\cite{AppleDP}, and Microsoft~\cite{DKY17}. Not surprisingly, these implementations require a huge user base to overcome the high error level.

The high level of noise required in the local model has motivated a recent search for alternative models. For example, the Encode, Shuffle, Analyze (ESA) model introduces a trusted shuffler that receives user messages and permutes them before they are handled to an untrusted analyzer~\cite{BEMMRLRKTS17}. A recent work by Cheu et al.~\cite{DBLP:journals/corr/abs-1808-01394} provides a formal analytical model for studying the shuffle model and protocols for summation of binary and real-valued inputs, essentially recovering the accuracy of the trusted curator model. The protocol for real-valued inputs requires users to send multiple messages, with a total of $O(\sqrt{n})$ single bit messages sent by each user. Also of relevance is the work of Ishai et al.~\cite{ikos} showing how to combine secret sharing with secure shuffling to implement distributed summation, as it allows to simulate the Laplace mechanism of the curator model. Instead we focus on the single-message shuffle model.

Another recent work by Erlingsson et al.~\cite{erlingsson2019amplification} shows that the shuffling primitive provides privacy amplification, as introducing random shuffling in local model protocols reduces $\varepsilon$ to $\varepsilon/\sqrt{n}$.

A word of caution is in place with respect to the shuffle model, as it differs significantly from the local model in terms of the assumed trust. In particular, the privacy guarantee provided by protocols in the shuffle model degrades with the fraction of users who deviate from the protocol. This is because, besides relying on a trusted shuffling step, the shuffle model requires users to provide messages carefully crafted to protect each other's privacy. This is in contrast with the curator model, where this responsibility is entirely held by the trusted curator.  Nevertheless, we believe that this model is of interest both for theoretical and practical reasons. On the one hand it allows to explore the space in between the local and curator model, and on the other hand it leads to mechanisms that are easy to explain, verify, and implement; with limited accuracy loss with respect to the curator model.

In this work we do not assume any particular implementation of the shuffling step. Naturally, alternative implementations will lead to different computational trade-offs and trust assumptions. 
The shuffle model allows to disentangle these aspects from the precise computation at hand, as the result of shuffling the randomized inputs submitted by each user is required to be differentially private, and therefore any subsequent analysis performed by the analyzer will be private due to the postprocessing property of differential privacy.

\subsection{Overview of Our Results}

In this work we focus on single-message shuffle model protocols. In such protocols (i) each user $i$ applies a local randomizer $\mech{R}$ on her input $x_i$ to obtain a single message $y_i$; (ii) the messages $(y_1,\ldots,y_n)$ are shuffled to obtain $(y_{\sigma(1)},\ldots, y_{\sigma(n)})$ where $\sigma$ is a randomly selected permutation; and (iii) an analyzer post-processes $(y_{\sigma(1)},\ldots, y_{\sigma(n)})$ to produce an outcome. It is required that the mechanism resulting from the combination of the local randomizer $\mech{R}$  and the random shuffle should provide differential privacy.

\subsubsection{A protocol for private summation.}

Our first contribution is a single-message shuffle model protocol for private summation of (real) numbers $x_i \in [0,1]$. The resulting estimator is unbiased and has standard deviation $O_{\varepsilon,\delta}(n^{1/6})$.  

To reduce the domain size, our protocol uses a fixed-point representation, where users apply randomized rounding to snap their input $x_i$ to a multiple $\bar{x}_i$ of $1/k$ (where $k = O_{\varepsilon,\delta}(n^{1/3})$). We then apply on $\bar{x}_i$ a local randomizer $\mech{R}^{PH}$ for computing private histograms over a finite domain of size $k+1$. The randomizer $\mech{R}^{PH}$ is simply a randomized response mechanism: with (small) probability $\gamma$ it ignores $\bar{x}_i$ and outputs a uniformly random domain element, otherwise it reports its input $\bar{x}_i$ truthfully. There are hence about $\gamma n$ instances of $\mech{R}^{PH}$ whose report is independent to their input, and whose role is to create what we call a \emph{privacy blanket}, which masks the outputs which are reported truthfully. Combining $\mech{R}^{PH}$ with a random shuffle, we get the equivalent of a histogram of the sent messages, which, in turn, is the pointwise sum of the histogram of approximately $(1-\gamma)n$ values  $\bar{x}_i$ sent truthfully and the privacy blanket, which is a histogram of approximately $\gamma n$ random values.

To see the benefit of creating a privacy blanket, consider the recent shuffle model summation protocol by Cheu et al.~\cite{DBLP:journals/corr/abs-1808-01394}. This protocol also applies randomized rounding. However, for privacy reasons, the rounded value needs to be represented in unary across multiple 1-bit messages, which are then fed into a summation protocol for binary values. The resulting error of this protocol is $O(1)$ (as is achieved in the curator model). However, the use of unary representation requires each user to send $O_{\varepsilon}(\sqrt{n})$ 1-bit messages (whereas in our protocol every user sends a single $O(\log n)$-bit message). We note that Cheu et al.\ also present a \emph{single message} protocol for real summation with $O(\sqrt{n})$ error.

\subsubsection{A lower bound for private summation.}

We also provide a matching lower bound showing that any single-message shuffled protocol for summation must exhibit mean squared error of order $\Omega(n^{1/3})$. In our lower bound argument we consider i.i.d.\ input distributions, for which we show that without loss of generality the local randomizer's image is the interval $[0,1]$, and the analyzer is a simple summation of messages. With this view, we can contrast the privacy and accuracy of the protocol. On the one hand, the randomizer may need to output $y\in[0,1]$ on input $x\in[0,1]$ such that $|x-y|$ is small, to promote accuracy. However, this interferes with privacy as it may enable distinguishing between the input $x$ and a potential input $x'$ for which $|x'-y|$ is large.

Together with our upper bound, this result shows that the single-message shuffle model sits strictly between the curator and the local models of differential privacy. This had been shown by Cheu et al.~\cite{DBLP:journals/corr/abs-1808-01394} in a less direct way by showing that (i) the private selection problem can be solved more accurately in the curator model than the shuffle model, and (ii) the private summation problem can be solved more accurately in the shuffle model than in the local model. For (i) they rely on a generic translation from the shuffle to the local model and known lower bounds for private selection in the local model, while our lower bound operates directly in the shuffle model. For (ii) they propose a single-message protocol that is less accurate than ours.

\subsubsection{Privacy amplification by shuffling.}\label{sec:results:amplification}

Lastly, we prove a new privacy amplification result for shuffled mechanisms. 
We show that shuffling $n$ copies of an $\varepsilon_0$-LDP local randomizer with $\varepsilon_0 = O(\log(n/\log(1/\delta)))$ yields an $(\varepsilon,\delta)$-DP mechanism with $\varepsilon = O((\varepsilon_0 \wedge 1) e^{\varepsilon_0}\sqrt{\log(1/\delta) /n})$, where $a \wedge b = \min\{a,b\}$. The proof formalizes the notion of a \emph{privacy blanket} that we use informally in the privacy analysis of our summation protocol. In particular, we show that the output distribution of local randomizers (for any local differentially private protocol) can be decomposed as a convex combination of an \emph{input-independent} blanket distribution and an \emph{input-dependent} distribution. 

Privacy amplification plays a major role in the design of differentially private mechanisms. These include amplification by subsampling~\cite{KLNRS08} and by iteration~\cite{FeldmanMTT18}, and the recent seminal work on amplification via shuffling by Erlingsson et al.~\cite{erlingsson2019amplification}.
In particular, Erlingsson et al.\ considered a setting more general than ours which allows for interactive protocols in the shuffle model by first generating a random permutation of the users' inputs and then sequentially applying a (possibly different) local randomizer to each element in the permuted vector. Moreover, each local randomizer is chosen depending on the output of previous local randomizers.
To distinguish this setting from ours, we shall call the setting of Erlingsson et al.\ \emph{shuffle-then-randomize} and ours \emph{randomize-then-shuffle}.
We also note that both settings are equivalent when there is a single local randomizer that will be applied to all the inputs.
Throughout this paper, unless we explicitly say otherwise, the term \emph{shuffle model} refers to the randomize-then-shuffle setting.

In the shuffle-then-randomize setting, Erlingsson et al.\ provide an amplification bound with $\varepsilon = O(\varepsilon_0 \sqrt{\log(1/\delta) /n})$ for $\varepsilon_0 = O(1)$.
Our result in the randomize-then-shuffle setting recovers this bound for the case of one randomizer, and extends it to $\varepsilon_0$ which is logarithmic in $n$.
For example, using the new bound, it is possible to shuffle a local randomizer with $\varepsilon_0 = O(\log(\varepsilon^2 n / \log(1/\delta)))$ to obtain a $(\varepsilon,\delta)$-DP mechanism with $\varepsilon = \Theta(1)$ .
Cheu et al.~\cite{DBLP:journals/corr/abs-1808-01394} also proved that a level of LDP $\varepsilon_0 = O(\log(\varepsilon^2 n / \log(1/\delta)))$ suffices to achieve $(\varepsilon,\delta)$-DP mechanisms through shuffling, though only for binary randomized response in the randomize-then-shuffle setting. 
Our amplification bound captures the regimes from both \cite{erlingsson2019amplification} and \cite{DBLP:journals/corr/abs-1808-01394}, thus providing a unified analysis of privacy amplification by shuffling for arbitrary local randomizers in the randomize-then-shuffle setting.
Our proofs are also conceptually simpler than those in \cite{erlingsson2019amplification,DBLP:journals/corr/abs-1808-01394} since we do not rely on privacy amplification by subsampling to obtain our results.

\section{Preliminaries}
\label{sec:preliminaries}
Our notation is standard. We denote domains as $\dom{X}$, $\dom{Y}$, $\dom{Z}$ and randomized mechanism as  $\mech{M}$, $\mech{P}$, $\mech{R}$, $\mech{S}$. For denoting sets and multisets we will use uppercase letters $A$, $B$, etc., and denote their elements as $a$, $b$, etc., while we will denote tuples as $\tup{x}$, $\tup{y}$, etc.
Random variables, tuples and sets are denoted by $\rv{X}$, $\tup{\rv{X}}$ and $\rvset{X}$ respectively. We also use greek letters $\mu$, $\nu$, $\omega$ for distributions.
Finally, we write $[k]=\{1,\ldots,k\}$, $a \wedge b = \min\{a,b\}$, $[u]_+ = \max\{u,0\}$ and $\mathbb{N}$ for the natural numbers.

\subsection{The Curator and Local Models of Differential Privacy}

\begin{figure}[t]

\begin{minipage}{0.45\textwidth}
\resizebox{\textwidth}{!}{
\centering
  \begin{tikzpicture}
    \node[draw, fill=white, thick, rounded corners, inner sep=2ex] (analyzer) {
      \textsc{Analyzer} $\mech{A}$
    };
    
    \node[draw, fill=white, rectangle, rounded corners, thick, align=center] at ($(analyzer)+(0, -4)$) (useri) {
      $y_i \gets \mech{R}(x_i)$\\[1ex]
      User $i$
    };
    \node[draw, fill=white, rectangle, rounded corners, thick, align=center] at ($(useri)+(-3, 0)$) (user1) {
      $y_1 \gets \mech{R}(x_1)$\\[1ex]
      User $1$
    };
    \node[draw, fill=white, rectangle, rounded corners, thick, align=center] at ($(useri)+(3, 0)$) (usern) {
      $y_n \gets \mech{R}(x_n)$\\[1ex]
      User $n$
    };
    \node[draw=none] at ($(useri)+(1.5, 0)$) (dots1) {
      $\ldots$
    };
    \node[draw=none] at ($(useri)+(-1.5, 0)$) (dots2) {
      $\ldots$
    };
    
     \draw[->, >=stealth, rounded corners, thick, dotted] ($(user1.north)+(0,0)$) -- ($(analyzer.south)+(-.35, 0)$) node[midway, sloped, above] {$y_1$};
     \draw[->, >=stealth, rounded corners, thick, dotted] ($(useri.north)+(0,0)$) -- ($(analyzer.south)+(0, 0)$) node[midway, sloped, above] {$y_i$};
     \draw[->, >=stealth, rounded corners, thick, dotted] ($(usern.north)+(0,0)$) -- ($(analyzer.south)+(.35, 0)$) node[midway, sloped, above] {$y_n$};

  \end{tikzpicture}
}
\end{minipage}
~~~~~~~~~~~~~~
\begin{minipage}{0.45\textwidth}
\resizebox{\textwidth}{!}{
\centering
  \begin{tikzpicture}
    \node[draw, fill=white, thick, rounded corners, inner sep=2ex] (analyzer) {
      \textsc{Analyzer} $\mech{A}$
    };

    \node[draw, fill=white, thick, rounded corners, inner sep=1ex] at ($(analyzer)+(0, -2)$) (shuffler) {
      \textsc{Shuffler} $\mech{S}$
    };
    
    \node[draw, fill=white, rectangle, rounded corners, thick, align=center] at ($(shuffler)+(0, -2)$) (useri) {
      $y_i \gets \mech{R}(x_i)$\\[1ex]
      User $i$
    };
    \node[draw, fill=white, rectangle, rounded corners, thick, align=center] at ($(useri)+(-3, 0)$) (user1) {
      $y_1 \gets \mech{R}(x_1)$\\[1ex]
      User $1$
    };
    \node[draw, fill=white, rectangle, rounded corners, thick, align=center] at ($(useri)+(3, 0)$) (usern) {
      $y_n \gets \mech{R}(x_n)$\\[1ex]
      User $n$
    };
    \node[draw=none] at ($(useri)+(1.5, 0)$) (dots1) {
      $\ldots$
    };
    \node[draw=none] at ($(useri)+(-1.5, 0)$) (dots2) {
      $\ldots$
    };
    
     \draw[->, >=stealth, rounded corners, thick] ($(user1.north)+(0,0)$) -- ($(shuffler.south)+(-.35, 0)$) node[midway, sloped, above] {$y_1$};
     \draw[->, >=stealth, rounded corners, thick] ($(useri.north)+(0,0)$) -- ($(shuffler.south)+(0, 0)$) node[midway, sloped, above] {$y_i$};
     \draw[->, >=stealth, rounded corners, thick] ($(usern.north)+(0,0)$) -- ($(shuffler.south)+(.35, 0)$) node[midway, sloped, above] {$y_n$};
    \draw[->, >=stealth, rounded corners, line width=.5ex, dotted] ($(shuffler.north)+(0,0)$) -- ($(analyzer.south)+(0, 0)$) node[midway, right] {$\mech{S}(y_1,\ldots,y_n)$};
  \end{tikzpicture}
  }
\end{minipage}

 \caption{The local (left) and shuffle (right) models of Differential Privacy. Dotted lines indicate differentially private values with respect to the dataset $\tup{x}=(x_1, \ldots, x_n)$, where user $i$ holds $x_i$.}
  \label{fig:local-shuffle-diagram}
\end{figure}
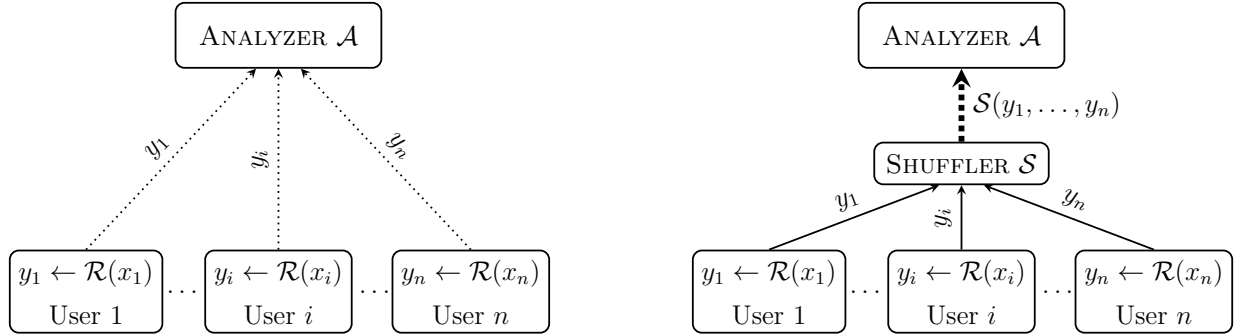

Differential privacy is a formal approach to privacy-preserving data disclosure that prevents attemps to learn private information about specific to individuals in a data release~\cite{DMNS06}. 
The definition of differential privacy requires that the contribution $x_i$ of an individual to a dataset $\tup{x}=(x_1, \ldots, x_n)$ has not much effect on what the adversary sees. This is formalized by considering a dataset $\tup{x}'$ that differs from $\tup{x}$ only in one element, denoted $\tup{x} \simeq \tup{x}'$, and requiring that the views of a potential adversary when running a mechanism on inputs $\tup{x}$ and $\tup{x}'$ are ``indistinguishable''. Let $\varepsilon\geq 0$ and $\delta\in[0,1]$. We say that a randomized mechanism $\mech{M} : \dom{X}^n \to \dom{Y}$ is $(\varepsilon,\delta)$-DP if
\begin{align*}
\forall \tup{x} \simeq \tup{x}', \forall E \subseteq \dom{Y} : \; \PP[\mech{M}(\tup{x}) \in E] \leq e^\varepsilon\PP[\mech{M}(\tup{x}') \in E] + \delta \enspace.
\end{align*}

As mentioned above, different models of differential privacy arise depending on whether one can assume the availability of a trusted party (a curator) that has access to the information from all users in a centralized location.
This setup is the one considered in the definition above.
The other extreme scenario is when each user privatizes their data locally and submits the private values to a (potentially untrusted) server for aggregation.
This is the domain of \emph{local} differential privacy\footnote{Of which, in this paper, we only consider the non-interactive version for simplicity.} (see Figure~\ref{fig:local-shuffle-diagram}, left), where a user owns a data record $x \in \dom{X}$ and uses a \emph{local randomizer} $\mech{R} : \dom{X} \to \dom{Y}$ to submit the privatized value $\mech{R}(x)$. In this case we say that the local randomizer is $(\varepsilon,\delta)$-LDP if
\begin{align*}
\forall x, x', \forall E \subseteq \dom{Y} : \; \PP[\mech{R}(x) \in E] \leq e^\varepsilon\PP[\mech{R}(x') \in E] + \delta \enspace.
\end{align*}
The key difference is that in this case we must protect each user's data, and therefore the definition considers changing a user's value $x$ to another arbitrary value $x'$.

Moving from curator DP to local DP can be seen as effectively redefining the view that an adversary has on the data during the execution of a mechanism.
In particular, if $\mech{R}$ is an $(\varepsilon,\delta)$-LDP local randomizer, then the mechanism $\mech{M} : \dom{X}^n \to \dom{Y}^n$ given by $\mech{M}(x_1,\ldots,x_n) = (\mech{R}(x_1), \ldots, \mech{R}(x_n))$ is $(\varepsilon,\delta)$-DP in the curator sense.
The single-message shuffle model sits in between these two settings.

\subsection{The Single-Message Shuffle Model}

The \emph{single-message shuffle model} of differential privacy considers a data collector that receives one message $y_i$ from each of the $n$ users as in the local model of differential privacy. The crucial difference with the local model is that the shuffle model assumes that a mechanism is in place to provide anonymity to each of the messages, i.e.\ the data collector is unable to associate messages to users. This is equivalent to assuming that, in the view of the adversary, these messages have been shuffled by a random permutation unknown to the adversary (see Figure~\ref{fig:local-shuffle-diagram}, right).

Following the notation in \cite{DBLP:journals/corr/abs-1808-01394}, we define a single-message protocol $\mech{P}$ in the shuffle model to be a pair of algorithms $\mech{P} = (\mech{R}, \mech{A})$, where $\mech{R}: \dom{X} \to \dom{Y}$, and $\mech{A}: \dom{Y}^n \to \dom{Z}$. We call $\mech{R}$ the \emph{local randomizer}, $\dom{Y}$ the \emph{message space} of the protocol, $\mech{A}$ the \emph{analyzer} of $\mech{P}$, and $\dom{Z}$ the \emph{output space}.
The overall protocol implements a mechanism $\mech{P} : \dom{X}^n \to \dom{Z}$ as follows.
Each user $i$ holds a data record $x_i$, to which she applies the local randomizer to obtain a message $y_i = \mech{R}(x_i)$.
The messages $y_i$ are then shuffled and submitted to the analyzer. We write $\mech{S}(y_1,\ldots,y_n)$ to denote the random shuffling step, where $\mech{S} : \dom{Y}^n \to \dom{Y}^n$ is a \emph{shuffler} that applies a random permutation to its inputs.
In summary, the output of $\mech{P}(x_1, \ldots, x_n)$ is given by $\mech{A} \circ \mech{S} \circ \mech{R}^n(\tup{x}) = \mech{A}(\mech{S}(\mech{R}(x_1), \ldots, \mech{R}(x_n)))$.

From a privacy point of view, our threat model assumes that the analyzer $\mech{A}$ is applied to the shuffled messages by an untrusted data collector.
Therefore, when analyzing the privacy of a protocol in the shuffle model we are interested in the indistinguishability between the shuffles $\mech{S} \circ \mech{R}^n(\tup{x})$ and $\mech{S} \circ \mech{R}^n(\tup{x}')$ for datasets $\tup{x} \simeq \tup{x}'$. 
In this sense, the analyzer's role is to provide utility for the output of the protocol $\mech{P}$, whose privacy guarantees follow from those of the \emph{shuffled mechanism} $\mech{M} = \mech{S} \circ \mech{R}^n : \dom{X}^n \to \dom{Y}^n$ by the post-processing property of differential privacy.
That is, the protocol $\mech{P}$ is $(\varepsilon,\delta)$-DP whenever the shuffled mechanism $\mech{M}$ is $(\varepsilon,\delta)$-DP.

When analyzing the privacy of a shuffled mechanism we assume the shuffler $\mech{S}$ is a perfectly secure primitive.
This implies that a data collector observing the shuffled messages $\mech{S}(y_1,\ldots,y_n)$ obtains no information about which user generated each of the messages.
An equivalent way to state this fact, which will sometimes be useful in our analysis of shuffled mechanisms, is to say that the output of the shuffler is a multiset instead of a tuple.
Formally, this means that we can also think of the shuffler as a deterministic map $\mech{S} : \dom{Y}^n \to \N_n^{\dom{Y}}$ which takes a tuple $\tup{y} = (y_1, \ldots, y_n)$ with $n$ elements from $\dom{Y}$ and returns the multiset $Y = \{y_1,\ldots, y_n\}$ of its coordinates, where $\N_n^{\dom{Y}}$ denotes the collection of all multisets over $\dom{Y}$ with cardinality $n$.
Sometimes we will refer to such multisets $Y \in \N_n^{\dom{Y}}$ as \emph{histograms} to emphasize the fact that they can be regarded functions $Y : \dom{Y} \to \N$ counting the number of occurrences of each element of $\dom{Y}$ in $Y$.

\subsection{Mean Square Error}

When analyzing the utility of shuffled protocols for real summation we will use the \emph{mean square error} (MSE) as accuracy measure.
The mean squared error of a randomized protocol $\mech{P}(\tup{x})$ for approximating a deterministic quantity $f(\tup{x})$ is given by $\mse(\mech{P},\tup{x})=\EE[(\mech{P}(\tup{x})-f(\tup{x}))^2]$, where the expectation is taken over the randomness of $\mech{P}$.
Note that when the protocol is unbiased the MSE is equivalent to the variance, since in this case we have $\Ex[\mech{P}(\tup{x})] = f(\tup{x})$ and therefore
\begin{align*}
\mse(\mech{P},\tup{x}) = \EE[(\mech{P}(\tup{x})-\Ex[\mech{P}(\tup{x})])^2] = \Var[\mech{P}(\tup{x})] \enspace.
\end{align*}

In addition to the MSE for a fixed input, we also consider the \emph{worst-case MSE} over all possible inputs $\mse(\mech{P})$, and the \emph{expected MSE} on a distribution over inputs $\mse(\mech{P}, \tup{\rv{X}})$. These quantities are defined as follows:
\begin{align*}
\mse(\mech{P}) &= \sup_{\tup{x}} \mse(\mech{P},\tup{x}) \enspace, \\
\mse(\mech{P}, \tup{\rv{X}}) &= \Ex_{\tup{x} \sim \tup{\rv{X}}}[ \mse(\mech{P},\tup{x}) ] \enspace.
\end{align*}
 
\section{The Privacy of Shuffled Randomized Response}
\label{sec:warmup}
In this section we show a protocol for $n$ parties to compute a private histogram over the domain $[k]$ in the single-message shuffle model.
The local randomizer of our protocol is shown in Algorithm~\ref{algo:lr-hist}, and the analyzer simply builds a histogram of the received messages. 
The randomizer is parameterized by a probability $\gamma$, and consists of a $k$-ary randomized response mechanism that returns the true value $x$ with probability $1-\gamma$,
and a uniformly random value with probability $\gamma$. This randomizer has been studied and used (in the local model) in several previous works~\cite{KairouzOV16, KairouzBR16, 2018arXiv181200984B}.
We discuss how to set $\gamma$ to satisfy differential privacy next.

\begin{algorithm2e}[t]
\DontPrintSemicolon
\SetKwComment{Comment}{{\scriptsize$\triangleright$\ }}{}
\SetKwInput{KwPub}{Public Parameters}
\caption{Private Histogram: Local Randomizer $\mech{R}^{PH}_{\gamma,k, n}$}\label{algo:lr-hist}
\KwPub{$\gamma\in [0, 1]$, domain size $k$, and number of parties $n$}
\KwIn{$x\in [k]$}
\KwOut{$y\in [k]$}
\BlankLine
Sample $b \gets \bernoulli(\gamma)$\;
\eIf{$b = 0$}{
Let $y \gets x$\;
}{
Sample $y \gets \uniform([k])$\;
}
\KwRet{$y$}
\end{algorithm2e}

\subsection{The \emph{Blanket} Intuition}

In each execution of Algorithm~\ref{algo:lr-hist} a subset $B$ of approximately $\gamma n$ parties
will submit a random value, while the remaining parties will submit their true value.  
The values sent by parties in $B$ form a histogram $Y_1$ of uniformly random values and the values sent by the parties not in $B$ correspond to the true histogram $Y_2$ of their data.
An important observation is that in the shuffle model the information obtained by the server is
equivalent to the histogram $Y_1 \cup Y_2$. 
This observation is a simple generalization of the observation made by Cheu et al.~\cite{DBLP:journals/corr/abs-1808-01394} that shuffling of binary data corresponds to secure addition.
When $k > 2$, shuffling of categorical data
corresponds to a secure histogram computation, and in particular secure addition of histograms.
In summary, the information collected by the server in an execution corresponds to
a histogram $Y$ with approximately $\gamma n$ random entries and $(1-\gamma)n$ truthful entries,
which as mentioned above we decompose as $Y = Y_1 \cup Y_2$.

To achieve differential privacy we need to set the value $\gamma$
of Algorithm~\ref{algo:lr-hist} so that $Y$
changes by an appropriately bounded amount when computed on neighboring datasets where only 
a certain party's data (say party $n$) changes. 
Our privacy argument does not rely on the anonymity of the set $B$ and thus we can assume, for the privacy analysis, that the server knows $B$.
We further assume in the analysis that the server knows the inputs from all parties except the $n$th one, which gives her the ability to remove from $Y$ the values submitted by any party who responded truthfully among the first $n-1$.

Now consider two datasets of size $n$ that differ on the input from the $n$th party.
In an execution where party $n$ is in $B$ we trivially get privacy since the value submitted by this party is independent of its input.
Otherwise, party $n$ will be submitting their true value $x_n$, in which case the server can determine
$Y_2$ up to the value $x_n$ using that she knows $(x_1, \ldots, x_{n-1})$.
Hence, a server trying to break the privacy of party $n$ observes $Y_1 \cup \{x_n\}$, the union of a random histogram with the input of this party.
Intuitively, the privacy of the protocol boils down to setting $\gamma$ so that $Y_1$, which we call the random \emph{blanket} of the local randomizer $\mech{R}^{PH}_{\gamma,k,n}$, appropriately ``hides'' $x_n$.

As we will see in Section~\ref{sec:amplification},
the intuitive notion of the blanket of a local randomizer can be
formally defined for arbitrary local randomizers using a generalization of
the notion of total variation distance from pairs to sets of distributions. This will allow us to
represent the output distribution of any local randomizer $\mech{R}(x)$ as a mixture of the form $(1-\gamma) \nu_x + \gamma \omega$, for some $0<\gamma<1$ and probability distributions $\nu_x$ and $\omega$, of which we call $\omega$ the \emph{privacy blanket} of the local randomizer $\mech{R}$.

\subsection{Privacy Analysis of Algorithm~\ref{algo:lr-hist}}

Let us now formalize the above intuition, and prove privacy for our protocol for an appropriate choice of $\gamma$.
In particular, we prove the following theorem, where the assumption $\varepsilon \leq 1$ is only for technical convenience.
A more general approach to obtain privacy guarantees for shuffled mechanisms is provided in Section~\ref{sec:amplification}.

\begin{theorem}\label{thm:priv_histogram}
The shuffled mechanism $\mech{M} = \mech{S} \circ \mech{R}^{PH}_{\gamma,k,n}$ is $(\varepsilon,\delta)$-DP for any $k, n \in \N$, $\varepsilon \leq 1$ and $\delta \in (0,1]$ such that $\gamma = \max\{\frac{14 k \log(2/\delta)}{(n-1) \varepsilon^2}, \frac{27 k}{(n-1) \varepsilon} \} < 1$.
\end{theorem}
\begin{proof}
Let $\D,\D'\in [k]^n$ be neighboring databases of the form $\D=(x_1,x_2,\ldots,x_n)$ and $\D'=(x_1,x_2,\ldots,x_n')$.
We assume that the server knows the set $B$ of users who submit random values, which is equivalent to revealing to the server a vector $\vec{b} = (b_1, \ldots, b_n)$ of the bits $b$ sampled in the execution of each of the local randomizers.
We also assume the server knows the inputs from the first $n-1$ parties.

Hence, we define the view $\view_{\mech{M}}$ of the server on a realization of the protocol as the tuple $\view_{\mech{M}}(\tup{x}) = (Y,\tup{x}_{\cap},\tup{b})$ containing:
\begin{enumerate}
\item A multiset $Y = \mech{M}(\tup{x}) = \{y_1, \ldots, y_n\}$ with the outputs $y_i$ of each local randomizer.
\item A tuple $\tup{x}_{\cap} = (x_1,\ldots,x_{n-1})$ with the inputs from the first $n-1$ users.
\item The tuple $\tup{b} = (b_1, \ldots, b_n)$ of binary values indicating which users submitted their true values.
\end{enumerate}
Proving that the protocol is $(\varepsilon,\delta)$-DP when the server has access to all this information will imply the same level of privacy for the shuffled mechanism $\mech{S} \circ \mech{R}^{PH}_{\gamma,k,n}$ by the post-processing property of differential privacy.

To show that $\view_{\mech{M}}$ satisfies $(\varepsilon,\delta)$-DP it is enough to prove
\begin{align*}
\Pr_{\rv{V} \sim \view_{\mech{M}}(\tup{x})} \left[ \frac{\Pr[\view_{\mech{M}}(\tup{x}) = \rv{V}]}{\Pr[\view_{\mech{M}}(\tup{x}') = \rv{V}]} \geq e^{\varepsilon} \right] \leq \delta \enspace.
\end{align*}

We start by fixing a value $V$ in the range of $\view_{\mech{M}}$ and computing the probability ratio above conditioned on $\rv{V} = V$.

Consider first the case where $V$ is such that $b_n = 1$, i.e.\ party $n$ submits a random value independent of her input. In this case privacy holds trivially since $\Pr[\view_{\mech{M}}(\tup{x}) = V] = \Pr[\view_{\mech{M}}(\tup{x}') = V]$.
Hence, we focus on the case where party $n$ submits her true value ($b_n = 0$).
For $j \in [k]$, let $n_j$ be the number of messages received by the server with value $j$ after removing from $Y$ any truthful answers submitted by the first $n - 1$ users.
With our notation above, we have $n_j = Y_1(j) + \one[x_n = j]$ and $\sum_{j=1}^{k} n_j = |B| + 1$ for the execution with input $\tup{x}$.
Now assume, without loss of generality, that $x_n = 1$ and $x_n' = 2$.
As $x_n = 1$, we have that
\begin{align*}
\Pr[\view_{\mech{M}}(\tup{x}) = V] &= \binom{|B|}{n_1-1,n_2,...,n_k}\frac{\gamma^{|B|}(1-\gamma)^{n-|B|}}{k^{|B|}} \enspace,
\end{align*}
corresponding to the probability of a particular pattern $\tup{b}$ of users sampling from the blanket times the probability of obtaining a particular histogram $Y_1$ when sampling $|B|$ elements uniformly at random from $[k]$.
Similarly, using that $x_n' = 2$ we have
\begin{align*}
\Pr[\view_{\mech{M}}(\tup{x}') = V] &= \binom{|B|}{n_1,n_2-1,...,n_k}\frac{\gamma^{|B|}(1-\gamma)^{n-|B|}}{k^{|B|}} \enspace.
\end{align*}
Therefore, taking the ratio between the last two probabilities we find that, in the case $b_n = 0$,
\begin{equation*}
\frac{\Pr[\view_{\mech{M}}(\tup{x}) = V]}{\Pr[\view_{\mech{M}}(\tup{x}') = V]}
=
\frac{n_1}{n_2} \enspace.
\end{equation*}

Now note that for $\rv{V} \sim \view_{\mech{M}}(\tup{x})$ the count $n_2 = n_2(\rv{V})$ follows a binomial distribution $\rv{N}_2$ with $n-1$ trials and success probability $\gamma/k$, and $n_1(\rv{V}) - 1 = \rv{N}_1 - 1$ follows the same distribution.
Thus, we have
\begin{align*}\label{eqn:binratio}
\Pr_{\rv{V} \sim \view_{\mech{M}}(\tup{x})} \left[ \frac{\Pr[\view_{\mech{M}}(\tup{x}) = \rv{V}]}{\Pr[\view_{\mech{M}}(\tup{x}') = \rv{V}]} \geq e^{\varepsilon} \right]
&=
\Pr\left[\frac{\rv{N}_1}{\rv{N_2}} \geq e^{\varepsilon}\right]
\enspace,
\end{align*}
where $\rv{N}_1 \sim \binomial\left(n-1, \frac{\gamma}{k}\right) + 1$ and $\rv{N}_2 \sim \binomial\left(n-1, \frac{\gamma}{k}\right)$.

We now bound the probability above using a union bound and the multiplicative Chernoff bound.
Let $c= \EE[\rv{N}_2] = \frac{\gamma (n-1)}{k}$.
Since $\rv{N}_1 / \rv{N_2} \geq e^{\varepsilon}$ implies that either $\rv{N}_1 \geq c e^{\varepsilon/2}$ or $\rv{N}_2 \leq c e^{-\varepsilon/2}$, we have
\begin{align*}
\Pr\left[\frac{\rv{N}_1}{\rv{N_2}} \geq e^{\varepsilon}\right]
&\leq
\Pr\left[\rv{N}_1 \geq c e^{\varepsilon/2}\right] + \Pr\left[\rv{N}_2 \leq c e^{-\varepsilon/2}\right] \\
&=
\Pr\left[\rv{N}_2 \geq c e^{\varepsilon/2} - 1\right] + \Pr\left[\rv{N}_2 \leq c e^{-\varepsilon/2} \right] \\
&=
\Pr\left[\rv{N}_2 - \Ex[\rv{N}_1] \geq c \left(e^{\varepsilon/2} - 1 - \frac{1}{c}\right) \right] \\
&\qquad + \Pr\left[\rv{N}_2 - \Ex[\rv{N}_2] \leq c (e^{-\varepsilon/2} - 1) \right] \enspace.
\end{align*}
Applying the multiplicative Chernoff bound to each of these probabilities then gives that
\begin{equation*}
\Pr\left[\frac{\rv{N}_1}{\rv{N_2}} \geq e^{\varepsilon}\right] \leq \exp\left(-\frac{c}{3} \left(e^{\varepsilon/2}-1 - \frac{1}{c}\right)^2 \right) + \exp\left(-\frac{c}{2} (1 - e^{-\varepsilon/2})^2 \right) \enspace.
\end{equation*}
Assuming $\varepsilon \leq 1$, both of the right hand summands are less than or equal to $\frac{\delta}{2}$ if 
\begin{equation*}
c = \frac{\gamma (n-1)}{k} \geq
\max\left\{\frac{14 \log\left(\frac{2}{\delta}\right)}{\varepsilon^2}, \frac{27}{\varepsilon}\right\} \enspace.
\end{equation*}
Indeed, for the second term this follows from $1 - e^{-\varepsilon/2} \geq (1 - e^{-1/2}) \varepsilon \geq \varepsilon/\sqrt{7}$ for $\varepsilon \leq 1$.
For the first term we use that $c \geq \frac{27}{\varepsilon}$ implies $e^{\varepsilon/2}-1 - \frac{1}{c} \geq \frac{25}{54} \varepsilon$ and $14 \geq \frac{3 \cdot 54^2}{25^2}$.
\end{proof}

Two remarks about this result are in order.
First, we should emphasize that the assumption of $\varepsilon \leq 1$ is only required for simplicity when using Chernoff's inequality to bound the probability that the privacy loss random variable is large.
Without any restriction on $\varepsilon$, a similar result can be achieved by replacing Chernoff's inequality with Bennett's inequality \cite[Theorem 2.9]{boucheron2013concentration} to account for the variance of the privacy loss random variable in the tail bound.
Here we decide not to pursue this route because the ad-hoc privacy analysis of Theorem~\ref{thm:priv_histogram} is superseded by the results in Section~\ref{sec:amplification} anyway.
The second observation about this result is that, with the choice of $\gamma$ made above, the local randomizer $\mech{R}^{PH}_{\gamma,k,n}$ satisfies $\varepsilon_0$-LDP with
\begin{align*}
\varepsilon_0 = O\left(\log \left(\frac{n \varepsilon^2}{\log(1/\delta)} - k \right) \right)
= O\left(\log \left(\frac{n \varepsilon^2}{\log(1/\delta)} \left(1 - \frac{\gamma}{14}\right) \right) \right) \enspace.
\end{align*}
This is obtained according to the formula provided by Lemma~\ref{lem:gamma_specific} in Section~\ref{sec:blanket}.
Thus, we see that Theorem~\ref{thm:priv_histogram} can be regarded as a privacy amplification statement showing that shuffling $n$ copies of an $\varepsilon_0$-LDP local randomized with $\varepsilon_0 = O_{\delta}(\log (n \varepsilon^2))$ yields a mechanism satisfying $(\varepsilon,\delta)$-DP.
In Section~\ref{sec:blanket} we will show that this is not coincidence, but rather an instance of a general privacy amplification result.

\section{Optimal Summation in the Shuffle Model}\label{sec:optimal_summation}

\subsection{Upper Bound}
\label{sec:upperbound}

In this section we present a protocol for the problem of computing the sum of real values $x_i\in [0,1]$ in the single-message shuffle model. Our protocol is parameterized by values $c,k$, and the number of parties $n$, and its local randomizer and analyzer are shown in Algorithms~\ref{algo:lr} and~\ref{algo:agg}, respectively.

\begin{algorithm2e}[t]
  \DontPrintSemicolon
  \SetKwInput{KwPub}{Public Parameters}
  \SetKwComment{Comment}{{\scriptsize$\triangleright$\ }}{}
\caption{Local Randomizer $\mech{R}_{c,k,n}$}\label{algo:lr}
        \KwPub{$c$, $k$, and number of parties $n$}
        \KwIn{$x \in [0,1]$}
        \KwOut{$y \in \{0,1,\ldots,k\}$}
\BlankLine
Let $\bar{x} \gets \lfloor xk \rfloor + \bernoulli(xk-\lfloor xk \rfloor)$ \Comment*[r]{$\bar{x}$ is the encoding of $x$ with precision $k$}
Sample $b\gets \bernoulli\left(\frac{c(k+1)}{n}\right)$\;
\eIf{$b = 0$}{
Let $y \gets \bar{x}$\;
}{
Sample $y \gets \uniform(\{0,1,\ldots,k\})$\;
}
\KwRet{$y$}
\end{algorithm2e}

\begin{algorithm2e}[t]
  \DontPrintSemicolon
    \SetKwInput{KwPub}{Public Parameters}
  \SetKwComment{Comment}{{\scriptsize$\triangleright$\ }}{}
\caption{Analyzer $\mech{A}_{c,k,n}$}\label{algo:agg}
        \KwPub{$c, k$, and number of parties $n$}
        \KwIn{Multiset $\{y_i\}_{i \in [n]}$, with $y _i\in \{0,1,\ldots,k\}$}
        \KwOut{$z \in [0,1]$}
\BlankLine
Let $\hat{z} \gets \frac{1}{k}\sum_{i=1}^n y_i$\;
Let $z \gets \debias(\hat{z})$, where $\debias(w) = \left(w - \frac{c(k+1)}{2} \right) / \left(1-\frac{c(k+1)}{n}\right)$\;
\KwRet{$z$}
\end{algorithm2e}

The protocol uses the protocol depicted in Algorithm~\ref{algo:lr-hist} in a black-box manner. To compute a differentially private approximation of $\sum_i x_i$, we fix a value $k$. Then we operate on the fixed-point encoding of each input $x_i$, which is an integer $\bar{x}_i \in\{0,\ldots,k\}$. That is, we replace $x_i$ with its fixed-point approximation $\bar{x}_i/k$.
The protocol then applies the randomized response mechanism in Algorithm~\ref{algo:lr-hist} to each $\bar{x}_i$ to submit a value $y_i$ to compute a differentially private histogram of the $(y_1,\ldots,y_n)$ as in the previous section. From these values the server can approximate $\sum_i x_i$ by post processing, which includes a debiasing standard step.
The privacy of the protocol described in Algorithms~\ref{algo:lr} and~\ref{algo:agg} follows directly from the privacy analysis of Algorithm~\ref{algo:lr-hist} given in Section~\ref{sec:warmup}.

Regarding accuracy, a crucial point in this reduction is that the encoding $\bar{x}_i$ of $x_i$ is via randomized rounding and hence unbiased. In more detail, as shown in Algorithm~\ref{algo:lr}, the value $x$ is encoded as $\bar{x} = \lfloor x k \rfloor + \bernoulli(xk-\lfloor xk \rfloor)$. This ensures that $\EE[\bar{x}/k] = \EE[x]$ and that the mean squared error due to rounding (which equals the variance) is at most $\frac{1}{4k^2}$.
The local randomizer either sends this fixed-point encoding or a random value in $\{0,1,\ldots,k\}$
with probabilities $1-\gamma$ and $\gamma$, respectively, where (following the analysis in the previous section) we set $\gamma = \frac{k+1}{n} c$.
Note that the mean squared error when the local randomizer submits a random value is at most $\frac{1}{2}$. This observations lead to the following accuracy bound.

\begin{theorem}
For any $\varepsilon \leq 1$, $\delta \in (0,1]$ and $n \in \N$,
  there exist parameters $c,k$ such that $\mech{P}_{c,k,n}$ is $(\varepsilon, \delta)$-DP and
\begin{align*}  
\mse(\mech{P}_{c,k,n})
=
O\left(n^{1/3} \cdot \frac{\log^{2/3}(1/\delta)}{\varepsilon^{4/3}}\right)
\enspace.
\end{align*}
\end{theorem}
\begin{proof}
The following bound on $\mse(\mech{P}_{c,k,n})$ follows from the observations above: unbiasedness of the estimator computed by the analyzer and randomized rounding, and the bounds on the variance of our randomized response.
\begin{align*}
\mse(\mech{P}_{c,k,n}) &= \sup_{\tup{x}}\EE[(\debias(\hat{z}) - \sum_{i} x_i)^2]\\
&= \sup_{\tup{x}}\EE\left[\left(\sum_i\left(\debias(y_i/k) - x_i\right)\right)^2\right]\\
&= \sup_{\tup{x}}\sum_i\EE\left[\left(\debias(y_i/k) - x_i\right)^2\right]\\
&= \sup_{\tup{x}}\sum_i\Var\left[\debias(y_i/k)\right]\\
&= \frac{n}{(1-\gamma)^2} \sup_{x_1} \Var[y_1/k]\\
&\leq \frac{n}{(1-\gamma)^2} \left(\frac{1-\gamma}{4 k^2} + \frac{\gamma}{2}\right)\\
&\leq \frac{n}{(1-\gamma)^2}\left(\frac{1}{4k^2} + \frac{c(k+1)}{2n}\right)
\enspace.
\end{align*}
Choosing the parameter $k = (n/c)^{1/3}$ minimizes the sum in the above expression and provides a bound on the $\mse$ of the form $O( c^{2/3} n^{1/3} )$.
Plugging in $c = \gamma \frac{n}{k+1} = O\left(\frac{\log(1/\delta)}{\varepsilon^2}\right)$ from our analysis in the previous section (Theorem~\ref{thm:priv_histogram}) yields the bound in the statement of the theorem.
\end{proof}

Note that as our protocol corresponds to an unbiased estimator, the $\mse$ is equal to the variance in this case. 
Using this observation we immediately obtain the following corollary for estimation of statistical queries in the single-message shuffle model.

\begin{corollary}
For every statistical query $q: {\mathcal X} \mapsto [0,1]$, $\varepsilon \leq 1, \delta \in (0,1]$ and $n\in\mathbb{N}$, there is an $(\varepsilon, \delta)$-DP $n$-party unbiased protocol for estimating $\frac{1}{n}\sum_i q(x_i)$ in the single-message shuffle model with standard deviation $O\left(\frac{\log^{1/3}(1/\delta)}{n^{5/6} \varepsilon^{2/3}}\right)$.
\end{corollary}

\subsection{Lower Bound}
\label{sec:lowerbound}

In this section we show that any differentially private protocol $\mech{P}$ for the problem of 
estimating $\sum_{i} x_i$
in the single-message shuffle model
must have $\mse(\mech{P}) = \Omega(n^{1/3})$
This shows that our protocol from the previous section is optimal, and gives a separation result for the single-message shuffle model, showing that its accuracy lies between the curator and local models of differential privacy.

\subsubsection{Reduction in the i.i.d.\ setting.}

We first show that when the inputs to the protocol $\mech{P}$ are sampled i.i.d.\ one can assume, for the purpose of showing a lower bound, that the protocol $\mech{P}$ for estimating $\sum_i  x_i$ is of a simplified form.
Namely, we show that the local randomizer can be taken to have output values in $[0,1]$, and its analyzer simply adds up all received messages.

\begin{lemma}\label{lemma:simplifying-assumption}
Let $\mech{P}=(\mech{R},\mech{A})$ be an $n$-party protocol for real summation in the single-message shuffle model.
Let $\rv{X}$ be a random variable on $[0,1]$ and suppose that users sample their inputs from the distribution $\tup{\rv{X}} = (\rv{X}_1,\ldots,\rv{X}_n)$, where each $\rv{X}_i$ is an independent copy of $\rv{X}$.
Then, there exists a protocol $\mech{P}'=(\mech{R}', \mech{A}')$ such that:
\begin{enumerate}
\item $\mech{A}'(y_1,\ldots,y_n)=\sum_{i=1}^n y_i$ and\footnote{Here we use $\image(\mech{R}')$ to denote the image of the local randomizer $\mech{R}'$.} $\image(\mech{R}')\subseteq [0,1]$.\label{it:Pprime}
\item $MSE(\mech{P}', \tup{\rv{X}})\leq MSE(\mech{P}, \tup{\rv{X}})$.\label{it:MSE_Pprime}
\item If the shuffled mechanism $\mech{S} \circ \mech{R}^n$ is $(\varepsilon,\delta)$-DP, then $\mech{S} \circ \mech{R}'^n$ is also $(\varepsilon,\delta)$-DP.\label{it:DP_Pprime}
\end{enumerate}
\end{lemma}
\begin{proof}
Consider the post-processed local randomizer $\mech{R}' = f \circ \mech{R}$ where $f(y) = \EE[\rv{X}|\mech{R}(\rv{X}) = y]$.
In Bayesian estimation, $f$ is called the posterior mean estimator, and is known to be a minimum MSE estimator \cite{jaynes2003probability}.
Since $\image(\mech{R}') \subseteq [0,1]$, we have a protocol $\mech{P}'$ satisfying claim \ref{it:Pprime}.

Next we show that $\mse(\mech{P}', \tup{\rv{X}})\leq \mse(\mech{P}, \tup{\rv{X}})$. 
Note that the analyzer $\mech{A}$ in protocol $\mech{P}$ can be seen as an estimator of $\rv{Z} = \sum_i \rv{X}_i$ given observations from $\tup{\rv{Y}} = (\rv{Y}_1, \ldots, \rv{Y}_n)$, where $\rv{Y}_i = \mech{R}(\rv{X}_i)$.
Now consider an arbitrary estimator $h$ of $\rv{Z}$ given the observation $\tup{\rv{Y}} = \tup{y}$. We have 
\begin{align*}
\mse(h, \tup{y}) & = \EE[(h(\tup{y}) - \rv{Z})^2 | \tup{\rv{Y}} = \tup{y}] \\
& = \EE[\rv{Z}^2 | \tup{\rv{Y}} = \tup{y}] -2 h(\tup{y}) \EE[\rv{Z}| \tup{\rv{Y}} = \tup{y}] + h(\tup{y})^2 \enspace.
\end{align*}
It follows from minimizing $\mse(h, \tup{y})$ with respect to $h$ that the minimum MSE estimator of $\rv{Z}$ given $\tup{\rv{Y}}$ is $h(\tup{y}) = \EE[\rv{Z} | \tup{\rv{Y}} = \tup{y}]$.
Hence, by linearity of expectation, and the fact that the $\rv{Y}_i$ are independent,
\begin{align*}
\EE[\rv{Z} | \tup{\rv{Y}} = \tup{y}]
&=
\sum_{i=1}^n \Ex[\rv{X}_i | \tup{\rv{Y}} = \tup{y}]
=
\sum_{i=1}^n \Ex[\rv{X}_i | \rv{Y}_i = y_i]
=
\sum_{i=1}^n f(y_i) \enspace.
\end{align*}
Therefore, we have shown that $\mech{P}' = (\mech{R}', \mech{A}')$ implements a minimum MSE estimator for $\rv{Z}$ given $(\mech{R}(\rv{X}_1), \ldots, \mech{R}(\rv{X}_n))$, and in particular $\mse(\mech{P}', \tup{\rv{X}})\leq \mse(\mech{P}, \tup{\rv{X}})$.

Part~\ref{it:DP_Pprime} of the lemma follows from the standard post-processing property of differential privacy by observing that the output of $\mech{S} \circ \mech{R}'^n (\tup{x})$ can be obtained by applying $f$ to each element in the output of $\mech{S} \circ \mech{R}^n(\tup{x})$.
\end{proof}

\subsubsection{Proof of the lower bound.} 

It remains to show that, for any protocol $\mech{P} = (\mech{R}, \mech{A})$ satisfying the conditions of Lemma~\ref{lemma:simplifying-assumption}, we can find a tuple of i.i.d.\ random variables $\tup{\rv{X}}$ such that $\mse(\mech{P},\tup{\rv{X}}) = \Omega(n^{1/3})$. 
Recall that by virtue of Lemma~\ref{lemma:simplifying-assumption} we can assume, without loss of generality, that $\mech{R}$ is a mapping from $[0,1]$ into itself, $\mech{A}$ sums its inputs, and $\tup{\rv{X}} = (\rv{X}_1,\ldots, \rv{X}_n)$ where the $\rv{X}_i$ are i.i.d.\ copies of some random variable $\rv{X}$.
We first show that under these assumptions we can reduce the search for a lower bound on $\mse(\mech{P}, \tup{\rv{X}})$ to consider only the expected square error of an individual run of the local randomizer.

\remove{
We can now argue about how frequently can it be that the difference $|\mech{R}(x)-x|$ is small (hence incurring little error) and how frequently difference needs to be large (hence incurring larger error). We will see how the latter case is necessary due to the privacy requirements.  
To gain intuition, consider the case when we have two neighboring datasets $D = (0, 1,\ldots,1)$ and $D' = (1, 1,\ldots,1)$. 
The privacy adversary tries to distinguish $\pi(\mech{R}(0), \mech{R}(1),\ldots, \mech{R}(1))$ from $\pi(\mech{R}(1), \mech{R}(1),\ldots, \mech{R}(1))$ for a secret random shuffle $\pi$. Note that, if $\mech{R}$ has a very small probability of sending a value close to $1$ when given $0$ as input, then it must be the case that $\mech{R}$ has significant probability of a value far from $1$ when given $1$ as input. Otherwise, the privacy adversary can easily distinguish whether the input database is $D$ or $D'$, as the probability of obtaining some value close to $0$ would be high on $D$ and small on $D'$. Intuitively, this means that for each input that incurs small error there must be an input incurring big error. 
}

\begin{lemma}\label{lem:mse-to-randomizer}
Let $\mech{P}=(\mech{R},\mech{A})$ be an $n$-party protocol for real summation in the single-message shuffle model such that $\mech{R} : [0,1] \to [0,1]$ and $\mech{A}$ is summation. Suppose $\tup{\rv{X}} = (\rv{X}_1,\ldots, \rv{X}_n)$, where the $\rv{X}_i$ are i.i.d.\ copies of some random variable $\rv{X}$. Then,
\begin{align*}
\mse(\mech{P}, \tup{\rv{X}}) \geq n \EE[(\mech{R}(\rv{X}) - \rv{X})^2] \enspace.
\end{align*}
\end{lemma}
\begin{proof}
The result follows from an elementary calculation:
\begin{align*}
\mse(\mech{P}, \tup{\rv{X}})
&=
\EE\left[\left(\sum_{i\in[n]}\mech{R}(\rv{X}_i) - \rv{X}_i\right)^2\right] \\
&=
\sum_i \EE[(\mech{R}(\rv{X}_i) - \rv{X}_i)^2]
+
\sum_{i \neq j}\EE[(\mech{R}(\rv{X}_i) - \rv{X}_i)(\mech{R}(\rv{X}_j) - \rv{X}_j)] \\
&=
\sum_i \EE[(\mech{R}(\rv{X}_i) - \rv{X}_i)^2] + \sum_{i\neq j} \EE[\mech{R}(\rv{X}_i) - \rv{X}_i]^2 \\
&\geq
n \EE[(\mech{R}(\rv{X}) - \rv{X})^2] \enspace.
\qedhere
\end{align*}
\end{proof}

Therefore, to obtain our lower bound it will suffice to find a distribution on $[0,1]$ such that if $\mech{R} : [0,1] \to [0,1]$ is a local randomizer for which the protocol $\mech{P} = (\mech{R}, \mech{A})$ is differentially private, then $\mech{R}$ has expected square error $\Omega(n^{-2/3})$ under that distribution.
We start by constructing such distribution and then show that it satisfies the desired properties.

Consider the partition of the unit interval $[0,1]$ into $k$ disjoint subintervals of size $1/k$, where $k \in \N$ is a parameter to be determined later.
We will take inputs from the set $\I = \{ m/k - 1/2k \; | \; m \in [k] \}$ of midpoints of these intervals.
For any $a \in \I$ we denote by $I(x)$ the subinterval of $[0,1]$ containing $a$.
Given a local randomizer $\mech{R} : [0,1] \to [0,1]$ we define the probability
$p_{a, b}= \PP[\mech{R}(a) \in I(b)]$ that the local randomizer maps an input $a$ to the subinterval centered at $b$ for any $a, b \in \I$.

Now let $\rv{X} \sim \uniform(\I)$ be a random variable sampled uniformly from $\I$.
The following observations are central to the proof of our lower bound.
First observe that $\mech{R}$ maps $\rv{X}$ to a value outside of its interval with probability $\frac{1}{k}\sum_{b \in \I} (1-p_{b,b})$.
If this event occurs, then $\mech{R}(\rv{X})$ incurs a squared error of at least $1/(2k)^2$, as the absolute error will be at least half the width of an interval. 
Similarly, when $\mech{R}$ maps an input $a$ to a point inside an interval $I(b)$ with $a \neq b$, the squared error incurred is at least $(|b-a|-1/2k)^2$, as the error is at least the distance between the two interval midpoints minus half the width of an interval.
The next lemma encapsulates a useful calculation related to this observation.

\begin{lemma}\label{lem:lb-calculation}
For any $b \in \I = \{ m/k - 1/2k \; | \; m \in [k] \}$ we have
\begin{align*}
\frac{1}{k} \sum_{a \in \I \setminus \{b\}} \left( |a-b|-\frac{1}{2k} \right)^2 \geq \frac{1}{48} \left(1 - \frac{1}{k^2}\right) \enspace.
\end{align*}
\end{lemma}
\begin{proof}
Let $b = m/k - 1/2k$ for some $m \in [k]$.
Then,
\begin{align*}
\frac{1}{k} \sum_{a \in \I \setminus \{b\}} \left( |a-b|-\frac{1}{2k} \right)^2
&=
\frac{1}{k^3} \sum_{i \in [k] \setminus \{m\}} \left( |i-m|-\frac{1}{2} \right)^2 \\
&\geq
\frac{1}{4 k^3} \sum_{i \in [k] \setminus \{m\}} (i-m)^2
=
\frac{1}{4 k^3} \sum_{i \in [k]} (i-m)^2
\enspace,
\end{align*}
where we used $(u - 1/2)^2 \geq u^2/4$ for $u \geq 1$.
Now let $\rv{U} \sim \uniform([k])$ and observe that for any $m \in [k]$ we have
\begin{equation*}
\sum_{i \in [k] } (i-m)^2 \geq \sum_{i \in [k]} (i - \Ex[\rv{U}])^2 = k \Var[\rv{U}] = \frac{k^3 - k}{12} \enspace.
\qedhere
\end{equation*}
\end{proof}

Now we can combine the two observations about the error of $\mech{R}$ under $\rv{X}$ into a lower bound for its expected square error.
Subsequently we will show how the output probabilities occurring in this bound are related under differential privacy.

\remove{
The general structure of the proof of the lower bound, given in the next two lemmas and final theorem, is as follows. We first prove, in Lemma~\ref{lemma:mse-bound}, a bound on $MSE(\mech{P},X)$, for $X\sim \omega(k)$, parametrised by $k$, $\sum_{b\in\I} \frac{1-p_{b,b}}{k}$, and $\sum_b \frac{\min_ap_{a,b}}{k}$. Next, in lemma~\ref{lemma:probs-bounds} we show that the requirement for privacy imposes a relationship between $p_{b,b}$ and $\min_a p_{a,b}$ that, once plugged into the previous bound, and for suitably chosen $k$, gets us the desired final bound on $\mse(\mech{P}, X)$ with $X\sim\omega(k)$.
}

\begin{lemma}\label{lem:randomizer-bound}
Let $\mech{R} : [0,1] \to [0,1]$ be a local randomizer and $\rv{X} \sim \uniform(I)$ with $\I = \{ m/k - 1/2k \; | \; m \in [k] \}$. Then,
\begin{equation*}
\EE[(\mech{R}(\rv{X}) - \rv{X})^2]
\geq
\sum_{b \in I} \min\left\{\frac{1-p_{b,b}}{4 k^3},  \frac{1}{48} \left(1 - \frac{1}{k^2}\right) \min_{a \in I} p_{a,b} \right\}
\enspace.
\end{equation*}
\end{lemma}
\begin{proof}
The bound in obtained by formalizing the two observations made above to obtain two different lower bounds for $\EE[(\mech{R}(\rv{X)} - \rv{X})^2]$ and then taking their minimum.
Our first bound follows directly from the discussion above:
\begin{align*}
\EE[(\mech{R}(\rv{X}) - \rv{X})^2]
&=
\sum_{b \in \I} \EE[(\mech{R}(b) - b)^2] \Pr[\rv{X} = b]
=
\frac{1}{k} \sum_{b \in \I} \EE[(\mech{R}(b) - b)^2] \\
&\geq
\frac{1}{k} \sum_{b \in \I} (1-p_{b,b}) \cdot \frac{1}{(2k)^2} 
=
\sum_{b\in \I} \frac{1 - p_{b,b}}{4 k^3} \enspace.
\end{align*}
Our second bound follows from the fact that the squared error is at least $(|b-a|-\frac{1}{2k})^2$ if $\rv{X} = a$ and $\mech{R}(a) \in I(b)$, for $a, b \in \I$ such that $a \neq b$:
\begin{align*}
\EE[(\mech{R}(\rv{X}) - \rv{X})^2]
&=
\frac{1}{k} \sum_{b \in \I} \EE[(\mech{R}(b) - b)^2] \\
&\geq
\frac{1}{k} \sum_{b \in \I}  \sum_{a \in \I \setminus \{b\}} p_{a,b} \left( |b-a|-\frac{1}{2k} \right)^2 \\
&\geq
\frac{1}{k} \sum_{b \in \I} (\min_{a \in \I} p_{a,b}) \sum_{a \in \I \setminus \{b\}} \left(|b-a|-\frac{1}{2k} \right)^2 \\
&\geq
\sum_{b \in \I} (\min_{a \in \I} p_{a,b}) \frac{1}{48} \left(1 - \frac{1}{k^2}\right)
\enspace,
\end{align*}
where the last inequality uses Lemma~\ref{lem:lb-calculation}.
Finally, we get
\begin{align*}
\EE[(\mech{R}(\rv{X}) - \rv{X})^2]
&\geq
\min\left\{
\sum_{b\in \I} \frac{1 - p_{b,b}}{4 k^3},
\sum_{b \in \I} (\min_{a \in \I} p_{a,b}) \frac{1}{48} \left(1 - \frac{1}{k^2}\right) \right\} \\
&\geq
\sum_{b\in \I} \min\left\{\frac{1 - p_{b,b}}{4 k^3}, \frac{1}{48} \left(1 - \frac{1}{k^2}\right) \min_{a \in I} p_{a,b} \right\} \enspace.
\qedhere
\end{align*}
\end{proof}

\begin{lemma}\label{lem:probs-bounds}
Let $\mech{R} : [0,1] \to [0,1]$ be a local randomizer such that the shuffled protocol $\mech{M} = \mech{S} \circ \mech{R}^n$ is $(\varepsilon,\delta)$-DP with $\delta < 1/2$.
Then, for any $a, b \in \I$, $a \neq b$, either $p_{b,b} < 1 - e^{-\varepsilon}/2$ or $p_{a,b} \geq (1/2 - \delta) / n$.
\end{lemma}
\begin{proof}
If $p_{b,b}< 1 - e^{-\varepsilon}/2$ then the proof is done. 
Otherwise, consider the neighboring datasets $\tup{x} = (a,\ldots,a)$ and $\tup{x}' = (b,a,\ldots,a)$.
Recall that the output of $\mech{M}(\tup{x})$ is the multiset obtained from the coordinates of $(\mech{R}(x_1), \ldots, \mech{R}(x_n))$.
By considering the event that this multiset contains no elements from $I(b)$, the definition of differential privacy gives
\begin{equation}\label{lemeq:dp-constraint}
\PP[\mech{M}(\tup{x}) \cap I(b) = \emptyset] \leq e^{\varepsilon} \PP[\mech{M}(\tup{x}') \cap I(b) = \emptyset] + \delta \enspace.
\end{equation}
As $\PP[\mech{M}(\tup{x}) \cap I(b) = \emptyset] = (1-p_{a,b})^n$ and $\PP[\mech{M}(\tup{x}') \cap I(b) = \emptyset] = (1-p_{b,b}) (1-p_{a,b})^{n-1} \leq (1-p_{b,b})$, we get from \eqref{lemeq:dp-constraint} that
\begin{equation*}
(1-p_{a,b})^n\leq(1-p_{b,b})e^\varepsilon+\delta \enspace.
\end{equation*}
As $p_{b,b} \geq 1- e^{-\varepsilon}/2$ we get that $p_{a,b} \geq 1- (1/2 + \delta)^{1/n}$ holds.
Finally, $p_{a,b} \geq (1/2-\delta)/n$ follows from the fact that
\begin{align*}
\left(1-\frac{1}{n}\left(\frac{1}{2}-\delta\right)\right)^n 
&= 1-\left(\frac{1}{2}-\delta\right)+\frac{n-1}{2n}\left(\frac{1}{2}-\delta\right)^2-\cdots \\
&\geq 1-\left(\frac{1}{2}-\delta\right)=\frac{1}{2}+\delta \enspace,
\end{align*}
which uses that the terms in the binomial expansion are alternating in sign and decreasing in magnitude.
\end{proof}

We can now choose $k = \ceil{n^{1/3}}$ and combine Lemmas~\ref{lem:mse-to-randomizer},~\ref{lem:randomizer-bound} and~\ref{lem:probs-bounds} to obtain our lower bound.

\begin{theorem}\label{thm:lower-bound}
Let $\mech{P}$ be an $(\varepsilon,\delta)$-DP $n$-party protocol for real summation on $[0,1]$ in the one-message shuffle model with $\delta < 1/2$.
Then, $\mse(\mech{P}) = \Omega(n^{1/3})$.
\end{theorem}
\begin{proof}
By the previous lemmas, taking $\tup{\rv{X}} = (\rv{X}_1, \ldots, \rv{X}_n)$ with independent $\rv{X}_i \sim \uniform(\I)$ we have
\begin{align*}
\mse(\mech{P},\tup{\rv{X}})
&\geq
n \sum_{b \in I} \min\left\{\frac{1-p_{b,b}}{4 k^3},  \frac{1}{48} \left(1 - \frac{1}{k^2}\right) \min_{a \in I} p_{a,b} \right\} \\
&\geq
n \sum_{b \in I} \min\left\{\frac{e^{-\varepsilon}}{8 k^3},  \frac{1}{48 n} \left(1 - \frac{1}{k^2}\right) \left(\frac{1}{2} - \delta\right) \right\} \\
&=
n k \min\left\{\frac{e^{-\varepsilon}}{8 k^3},  \frac{1}{48 n} \left(1 - \frac{1}{k^2}\right) \left(\frac{1}{2} - \delta\right) \right\} \enspace.
\end{align*}
Therefore, taking $k = \ceil{n^{1/3}}$ yields $\mse(\mech{P},\tup{\rv{X}}) = \Omega(n^{1/3})$.
Finally, the result follows from observing that a lower bound for the expected MSE implies a lower bound for worst-case MSE:
\begin{equation*}
\mse(\mech{P}) = \sup_{\tup{x} \in [0,1]^n} \mse(\mech{P}, \tup{x}) \geq \sup_{\tup{x} \in \I^n} \mse(\mech{P}, \tup{x})
\geq \mse(\mech{P},\tup{\rv{X}}) = \Omega(n^{1/3}) \enspace.
\qedhere
\end{equation*}
\end{proof}

\section{Privacy Amplification by Shuffling}
\label{sec:amplification}

In this section we prove a new privacy amplification result for shuffled mechanisms.
In particular, we will show that shuffling $n$ copies of an $\varepsilon_0$-LDP local randomizer with $\varepsilon_0 = O(\log(n/\log(1/\delta)))$ yields an $(\varepsilon,\delta)$-DP mechanism with $\varepsilon = O((\varepsilon_0 \wedge 1) e^{\varepsilon_0}\sqrt{\log(1/\delta) /n})$, where $a \wedge b = \min\{a,b\}$.
For this same problem, the following privacy amplification bound was obtained by Erlingsson et al.\ in \cite{erlingsson2019amplification}, which we state here for the randomize-then-shuffle setting (cf.\ Section~\ref{sec:results:amplification}).

\begin{theorem}[\cite{erlingsson2019amplification}]\label{thm:efmrtt}
If $\mech{R}$ is a $\varepsilon_0$-LDP local randomizer with $\varepsilon_0 < 1/2$, then the shuffled protocol $\mech{S} \circ \mech{R}^n$ is $(\varepsilon, \delta)$-DP with
\begin{align*}
\varepsilon = 12 \varepsilon_0 \sqrt{\frac{\log(1/\delta)}{n}}
\end{align*}
for any $n \geq 1000$ and $\delta < 1/100$.
\end{theorem}

Note that our result recovers the same dependencies on $\varepsilon_0$, $\delta$ and $n$ in the regime $\varepsilon_0 = O(1)$.
However, our bound also shows that privacy amplification can be extended to a wider range of parameters.
In particular, this allows us to show that in order to design a shuffled $(\varepsilon,\delta)$-DP mechanism with $\varepsilon = \Theta(1)$ it suffices to take any $\varepsilon_0$-LDP local randomizer with $\varepsilon_0 = O(\log(\varepsilon^2 n / \log(1/\delta)))$.
For shuffled binary randomized response, a dependence of the type $\varepsilon_0 = O(\log(\varepsilon^2 n / \log(1/\delta)))$ between the local and central privacy parameters was obtained in \cite{DBLP:journals/corr/abs-1808-01394} using an ad-hoc privacy analysis.
Our results show that this amplification phenomenon is not intrinsic to binary randomized response, and in fact holds for any pure LDP local randomizer.
Thus, our bound captures the privacy amplification regimes from both \cite{erlingsson2019amplification} and \cite{DBLP:journals/corr/abs-1808-01394}, thus providing a unified analysis of privacy amplification by shuffling.

To prove our bound, we first generalize the key idea behind the analysis of shuffled randomized response given in Section~\ref{sec:warmup}.
This idea was to ignore any users who respond truthfully, and then show that the responses of users who respond randomly provide privacy for the response submitted by a target individual.
To generalize this approach beyond randomized response we introduce the notions of \emph{total variation similarity} $\gamma_{\mech{R}}$ and \emph{blanket distribution} $\omega_{\mech{R}}$ of a local randomizer $\mech{R}$.
The similarity $\gamma_{\mech{R}}$ measures the probability that the local randomizer will produce an output that is independent of the input data. When this happens, the mechanism submits a sample from the blanket probability distribution $\omega_{\mech{R}}$.
In the case of Algorithm~\ref{algo:lr-hist} in Section~\ref{sec:warmup}, the parameter $\gamma_{\mech{R}^{PH}}$ is the probability $\gamma$ of ignoring the input and submitting a sample from $\omega_{\mech{R}^{PH}} = \uniform([k])$, the uniform distribution on $[k]$.
We define these objects formally in Section~\ref{sec:blanket}, then give further examples and also study the relation between $\gamma_{\mech{R}}$ and the privacy guarantees of $\mech{R}$.

The second step of the proof is to extend the argument that allows us to ignore the users who submit truthful responses in the privacy analysis of randomized response. In the general case, with probability $1 - \gamma_{\mech{R}}$ the local randomizer's outcome depends on the data but is not necessarily deterministic.
Analyzing this step in full generality -- where the randomizer is arbitrary and the domain might be uncountable -- is technically challenging.
We address this challenge by leveraging a characterization of differential privacy in terms of hockey-stick divergences that originated in the formal methods community to address the verification for differentially private programs \cite{barthe2013beyond,BartheKOB12,BartheGGHS16} and has also been used to prove tight results on privacy amplification by subsampling \cite{DBLP:conf/nips/BalleBG18}.
As a result of this step we obtain a privacy amplification bound in terms of the expectation of a function of a sum of i.i.d.\ random variables.
Our final bound is obtained by using a concentration inequality to bound this expectation.

The bound we obtain with this method provides a relation of the form $F(\varepsilon,\varepsilon_0,\gamma,n) \leq \delta$, where $F$ is a complicated non-linear function.
By simplifying this function $F$ further we obtain the asymptotic amplification bounds sketched above, where a bound for $\gamma$ in terms of $\varepsilon_0$ is used.
One can also obtain better mechanism-dependent bounds by computing the exact $\gamma$ for a given mechanism.
In addition, fixing all but one of the parameters of the problem we can numerically solve the inequality $F(\varepsilon,\varepsilon_0,\gamma,n) \leq \delta$ to obtain exact relations between the parameters without having to provide appropriate constants for the asymptotic bounds in closed-form.
We experimentally showcase the advantages of this approach to privacy calibration in Section~\ref{sec:experiments}.

Proofs for every result stated in this section are provided in Appendix~\ref{sec:proofs}.

\subsection{Blanket Decomposition}\label{sec:blanket}

The goal of this section is to provide a canonical way of decomposing any local randomizer $\mech{R} : \dom{X} \to \dom{Y}$ as a mixture between an input-dependent and an input-independent mechanism.
More specifically, let $\mu_x$ denote the output distribution of $\mech{R}(x)$.
Given a collection of distributions $\{\mu_x\}_{x \in \dom{X}}$ we will show how to find a probability $\gamma$, a distribution $\omega$ and a collection of distribution $\{\nu_x\}_{x \in \dom{X}}$ such that for every $x \in \dom{X}$ we have the mixture decomposition $\mu_x = (1 - \gamma) \nu_x + \gamma \omega$.
Since the component $\omega$ does not depend on $x$, this decomposition shows that $\mech{R}(x)$ is input oblivious with probability $\gamma$.
Furthermore, our construction provides the largest possible $\gamma$ for which this decomposition can be attained. 

To motivate the construction sketched above it will be useful to recall a well-known property of the \emph{total variation distance}. Given probability distributions $\mu, \mu'$ over $\dom{Y}$, this distance is defined as
\begin{align*}
\TV(\mu \| \mu') = \sup_{E \subseteq \dom{Y}} \left(\mu(E) - \mu'(E)\right) = \frac{1}{2} \int |\mu(y) - \mu'(y)| dy \enspace.
\end{align*}
Note how here we use the notation $\mu(y)$ to denote the ``probability'' of an individual outcome, which formally is only valid when the space $\dom{Y}$ is discrete so that every singleton is an atom.
Thus, in the case where $\dom{Y}$ is a continuous space we take $\mu(y)$ to denote the density of $\mu$ at $y$, where the density is computed with respect to some base measure on $\dom{Y}$.
We note that this abuse of notation is introduced for convenience and does not restrict the generality of our results.

The total variation distance admits a number of alternative characterizations. The following one is particularly useful:
\begin{align}\label{eqn:tvmin}
\TV(\mu \| \mu') = 1 - \int \min\{\mu(y), \mu'(y)\} dy \enspace.
\end{align}
This shows that $\TV(\mu \| \mu')$ can be computed in terms of the total probability mass that is simultaneously under $\mu$ and $\mu'$.
Equation \ref{eqn:tvmin} can be derived from the interpretation of the total variation distance in terms of couplings \cite{lindvall2002lectures}.
Using this characterization it is easy to construct mixture decompositions of the form $\mu = (1 - \gamma) \nu + \gamma \omega$, $\mu' = (1 - \gamma) \nu' + \gamma \omega$, where $\gamma = 1 - \TV(\mu \| \mu')$ and
$\omega(y) = \min\{\mu(y), \mu'(y)\} / \gamma$.
These decompositions are optimal in the sense that $\gamma$ is maximal and $\nu$ and $\nu'$ have disjoint support.

Extending the ideas above to the case with more than two distributions will provide the desired decomposition for any local randomizer.
In particular, we define the \emph{total variation similarity} of a set of distributions $\Lambda = \{\mu_x\}_{x \in \dom{X}}$ over $\dom{Y}$ as
\begin{align*}
\gamma_\Lambda = \int \inf_{x} \mu_x(y) dy \enspace.
\end{align*}
We also define the \emph{blanket distribution} of $\Lambda$ as the distribution given by $\omega_{\Lambda}(y) = \inf_{x} \mu_x(y) / \gamma_\Lambda$.
In this way, given a set of distributions $\Lambda = \{\mu_x\}_{x \in \dom{X}}$ with total variation similarity $\gamma$ and blanket distribution $\omega$, we obtain a mixture decomposition $\mu_x = (1 - \gamma) \nu_x + \gamma \omega$ for each distribution in $\Lambda$, where it is immediate to check that $\nu_x = (\mu_x - \gamma \omega) / (1 - \gamma)$ is indeed a probability distribution.
It follows from this construction that $\gamma$ is maximal since one can show that, by the definition of $\omega$, for each $y$ there exists an $x$ such that $\nu_x(y) = 0$. Thus, it is not possible to increase $\gamma$ while ensuring that $\nu_x$ are probability distributions.

\begin{figure}[t]
\centering
\includegraphics[height=.45\textwidth]{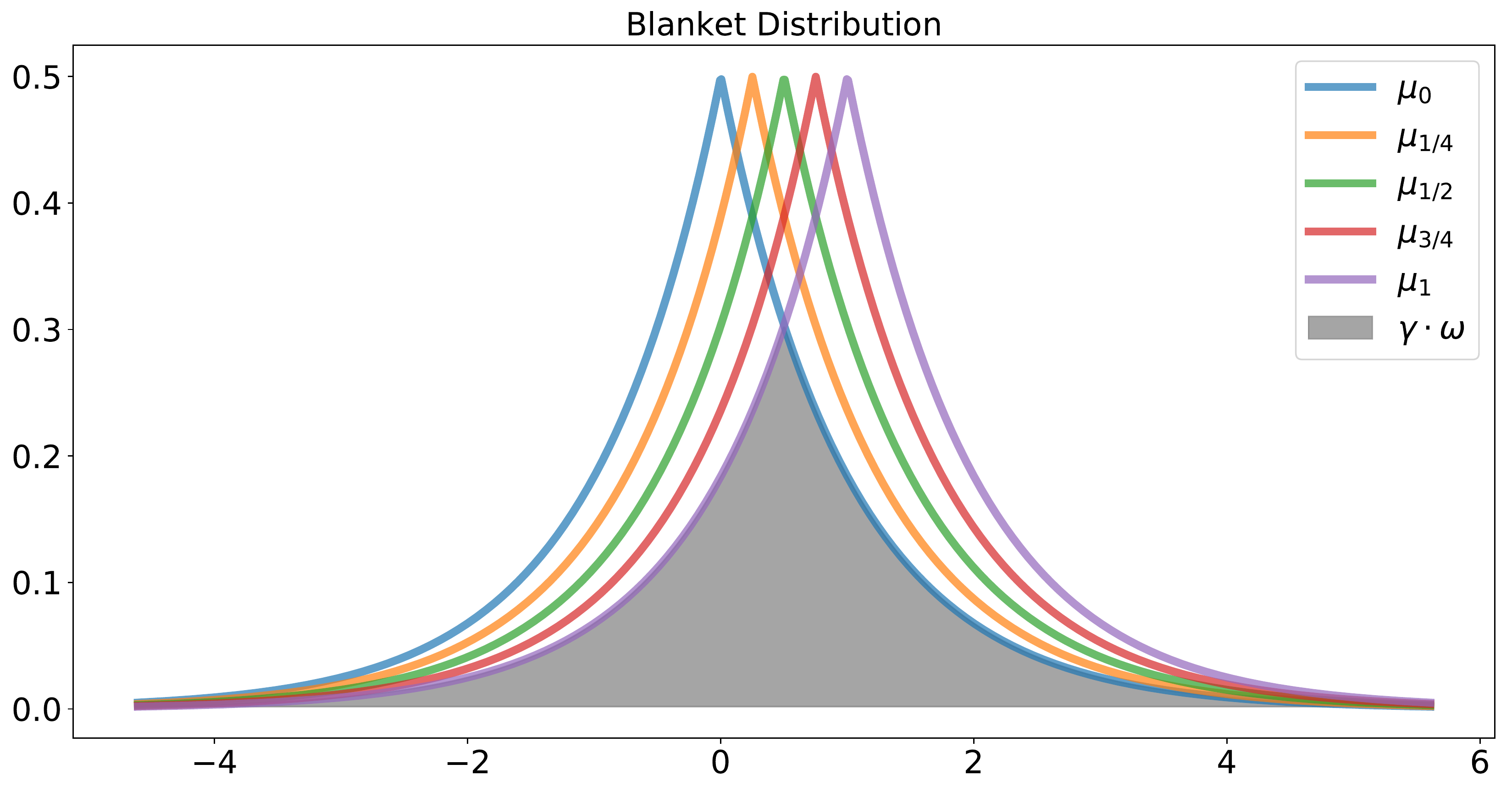}
\includegraphics[height=.45\textwidth]{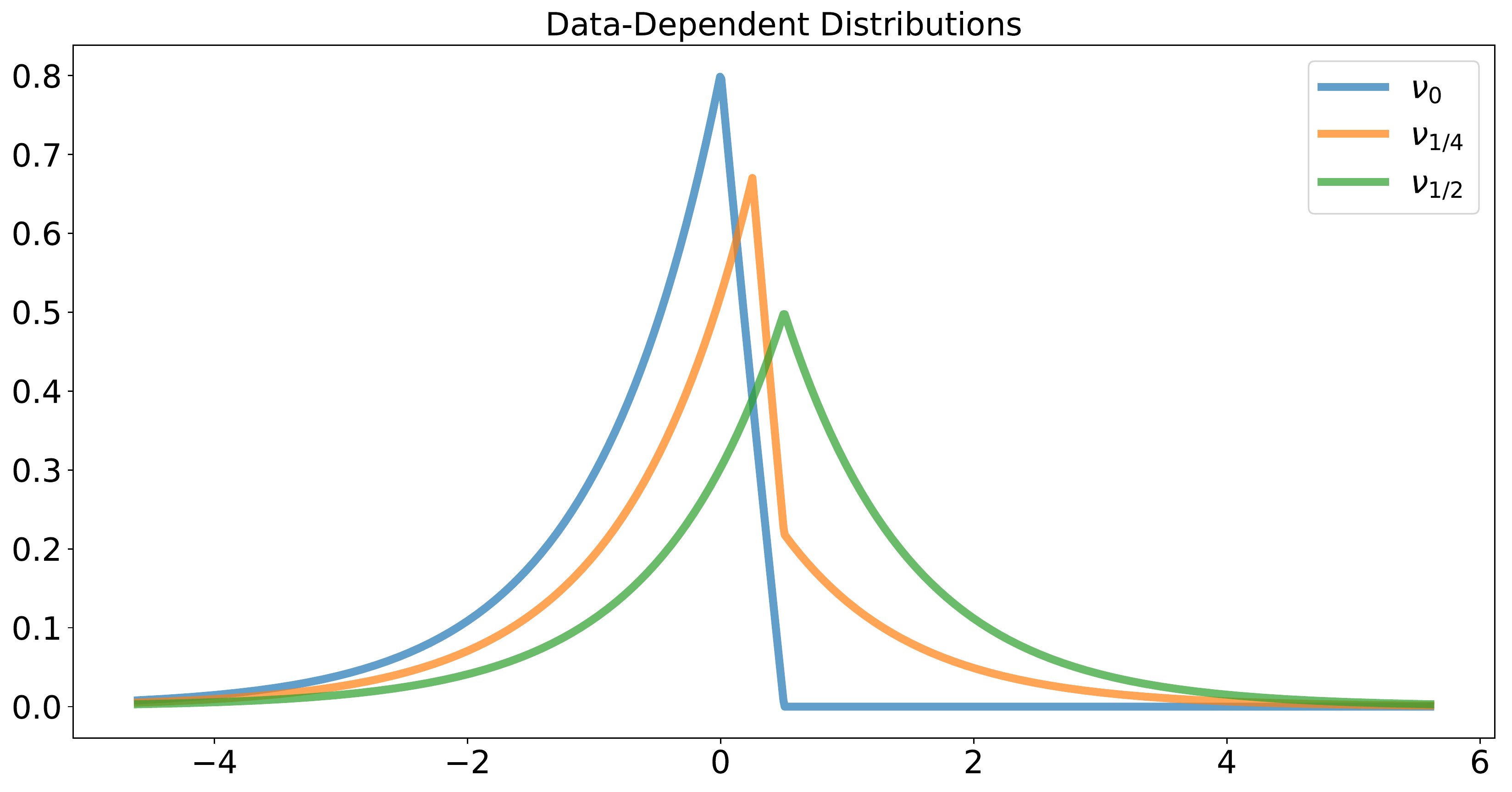}
\caption{Illustration of the blanket distribution $\omega$ and the data-dependent distributions $\nu_{x}$ corresponding to a $1$-LDP Laplace mechanism with inputs on $[0,1]$.}
\label{fig:blanket}
\end{figure}

Accordingly, we can identify a local randomizer $\mech{R}$ with the set of distributions $\{\mech{R}(x)\}_{x \in \dom{X}}$ and define the total variation similarity $\gamma_{\mech{R}}$ and the blanket distributions $\omega_{\mech{R}}$ of the mechanism. As usual, we shall just write $\gamma$ and $\omega$ when the randomizer is clear from the context.
Figure~\ref{fig:blanket} plots the blanket distribution and the data-dependent distributions corresponding to the local randomizer obtained by the Laplace mechanism with inputs on $[0,1]$.

The next result provides expressions for the total variation similarity of three important randomizers: $k$-ary randomized response, the Laplace mechanism on $[0,1]$ and the Gaussian mechanism on $[0,1]$.
Note that two of these randomizers offer pure LDP while the third one only offers approximate LDP, showing that the notion of total variation similarity and blanket distribution are widely applicable.

\begin{lemma}\label{lem:gamma_specific}
The following hold:
\begin{enumerate}
\item $\gamma = k / (e^{\varepsilon_0} + k - 1)$ for $\varepsilon_0$-LDP randomized response on $[k]$,\label{it:gamma_rr}
\item $\gamma = e^{-\varepsilon_0/2}$ for $\varepsilon_0$-LDP Laplace on $[0,1]$,\label{it:gamma_lap}
\item $\gamma = 2 \Pr[\Normal(0,\sigma^2) \leq -1/2]$ for a Gaussian mechanism with variance $\sigma^2$ on $[0,1]$.\label{it:gamma_gauss}
\end{enumerate}
\end{lemma}

This lemma illustrates how the privacy parameters of a local randomizer and its total variation similarity are related in concrete instances.
As expected, the probability of sampling from the input-independent blanket grows as the mechanisms become more private.
For arbitrary $\varepsilon_0$-LDP local randomizers we are able to show that the probability $\gamma$ of ignoring the input is at least $e^{-\varepsilon_0}$.

\begin{lemma}\label{lem:gamma_generic}
The total variation similarity of any $\varepsilon_0$-LDP local randomizer satisfies $\gamma \geq e^{-\varepsilon_0}$.
\end{lemma}

\subsection{Privacy Amplification Bounds}\label{sec:amplification_bounds}

Now we proceed to prove the amplification bound stated at the beginning of Section~\ref{sec:amplification}.
The key ingredient in this proof is to reduce the analysis of the privacy of a shuffled mechanism to the problem of bounding a function of i.i.d.\ random variables.
This reduction is obtained by leveraging the characterization of differential privacy in terms of hockey-stick divergences.

Let $\mu, \mu'$ be distributions over $\dom{Y}$. The \emph{hockey-stick divergence} of order $e^{\varepsilon}$ between $\mu$ and $\mu'$ is defined as
\begin{align*}
\div_{e^{\varepsilon}}(\mu \| \mu') = \int [\mu(y) - e^{\varepsilon} \mu'(y)]_+ dy \enspace,
\end{align*}
where $[u]_+ = \max\{0,u\}$.
Using these divergences one obtains the following useful characterization of differential privacy.

\begin{theorem}[\cite{barthe2013beyond}]
A mechanism $\mech{M} : \dom{X}^n \to \dom{Y}$ is $(\varepsilon,\delta)$-DP if and only if $\div_{e^{\varepsilon}}(\mech{M}(\tup{x}) \| \mech{M}(\tup{x}')) \leq \delta$ for any $\tup{x} \simeq \tup{x}'$.
\end{theorem}

This result is straightforward once one observes the identity
\begin{align*}
\int [\mu(y) - e^{\varepsilon} \mu'(y)]_+ dy = \sup_{E \subseteq \dom{Y}} \left(\mu(E) - e^{\varepsilon} \mu'(E)\right) \enspace.
\end{align*}

An important advantage of the integral formulation is that enables one to reason over individual outputs as opposed to sets of outputs for the case of $(\varepsilon,\delta)$-DP.
This is also the case for the usual sufficient condition for $(\varepsilon,\delta)$-DP in terms of a high probability bound for the privacy loss random variable.
However, this sufficient condition is not tight for small values of $\varepsilon$ \cite{BWicml17agm}, so here we prefer to work with the divergence-based characterization.

The first step in our proof of privacy amplification by shuffling is to provide a bound for the divergence $\div_{e^{\varepsilon}}(\mech{M}(\tup{x}) \| \mech{M}(\tup{x}'))$ for a shuffled mechanism $\mech{M} = \mech{S} \circ \mech{R}^n$ in terms of a random variable that depends on the blanket of the local randomizer.
Let $\mech{R} : \dom{X} \to \dom{Y}$ be a local randomizer with blanket $\omega$.
Suppose $\rv{W} \sim \omega$ is a $\dom{Y}$-valued random variable sampled from the blanket.
For any $\varepsilon \geq 0$ and $x, x' \in \dom{X}$ we define the \emph{privacy amplification random variable} as
\begin{align*}
\rv{L}_{\varepsilon}^{x,x'} = \frac{\mu_{x}(\rv{W}) - e^{\varepsilon} \mu_{x'}(\rv{W})}{\omega(\rv{W})} \enspace,
\end{align*}
where $\mu_x$ (resp.\ $\mu_{x'}$) is the output distribution of $\mech{R}(x)$ (resp.\ $\mech{R}(x')$).
This definition allows us to obtain the following result.

\begin{lemma}\label{lem:sumL}
Let $\mech{R} : \dom{X} \to \dom{Y}$ be a local randomizer and let $\mech{M} = \mech{S} \circ \mech{R}^n$ be the shuffling of $\mech{R}$.
Fix $\varepsilon \geq 0$ and inputs $\tup{x} \simeq \tup{x}'$ with $x_n \neq x'_n$.
Suppose $\rv{L}_1, \rv{L}_2, \ldots$ are i.i.d.\ copies of $\rv{L}_{\varepsilon}^{x,x'}$ and $\gamma$ is the total variation similarity of $\mech{R}$. 
Then we have the following:
\begin{align}\label{eqn:lem:sumL}
\div_{e^{\varepsilon}}(\mech{M}(\tup{x}) \| \mech{M}(\tup{x}')) \leq
\frac{1}{\gamma n} \sum_{m=1}^n \binom{n}{m} \gamma^m (1-\gamma)^{n-m}
\Ex\left[\sum_{i=1}^m \rv{L}_i\right]_+ \enspace.
\end{align}
\end{lemma}

The bound above can also be given a more probabilistic formulation as follows.
Let $\rv{M} \sim \binomial(n,\gamma)$ be the random variable counting the number of users who sample from the blanket of $\mech{R}$. Then we can re-write \eqref{eqn:lem:sumL} as
\begin{align*}
\div_{e^{\varepsilon}}(\mech{M}(\tup{x}) \| \mech{M}(\tup{x}')) \leq
\frac{1}{\gamma n}
\Ex\left[\sum_{i=1}^{\rv{M}} \rv{L}_i\right]_+ \enspace,
\end{align*}
where we use the convention $\sum_{i=1}^m \rv{L}_i = 0$ when $m = 0$.

Leveraging this bound to analyze the privacy of a shuffled mechanism requires some information about the privacy amplification random variables of an arbitrary local randomizer.
The main observation here is that $\rv{L}_\varepsilon^{x,x'}$ has negative expectation.
This means we can expect $\Ex[\sum_{i=1}^m \rv{L}_i]_+$ to decrease with $m$ since adding more variables will shift the expectation of $\sum_{i=1}^m \rv{L}_i$ towards $-\infty$, thus making it less likely to be above $0$.
Since $m$ represents the number of users who sample from the blanket, this reinforces the intuition that having more users sample from the blanket makes it easier for the data of the $n$th user to be hidden among these samples.
The following lemma will help us make this precise by providing the expectation of $\rv{L}_\varepsilon^{x,x'}$ as well as its range and second moment.

\begin{lemma}\label{lem:Labc}
Let $\mech{R} : \dom{X} \to \dom{Y}$ be an $\varepsilon_0$-LDP local randomizer with total variation similarity $\gamma$. For any $\varepsilon \geq 0$ and $x, x' \in \dom{X}$ the privacy amplification random variable $\rv{L} = \rv{L}_\varepsilon^{x,x'}$ satisfies:
\begin{enumerate}
\item $\Ex\rv{L} = 1 - e^{\varepsilon}$,\label{it:expL}
\item $\gamma e^{-\varepsilon_0} (1 - e^{\varepsilon+2\varepsilon_0}) \leq \rv{L} \leq \gamma e^{\varepsilon_0} (1 - e^{\varepsilon-2\varepsilon_0})$,\label{it:rangeL}
\item $\Ex\rv{L}^2 \leq \gamma e^{\varepsilon_0}(e^{2 \varepsilon} +1) - 2 \gamma^2 e^{\varepsilon-2\varepsilon_0}$.\label{it:varL}
\end{enumerate}
\end{lemma}

Now we can use the information about the privacy amplification random variables of an $\varepsilon_0$-LDP local randomizer provided by the previous lemma to give upper bounds for $\Ex[\sum_{i=1}^m \rv{L}_i]_+$.
This can be achieved by using concentration inequalities to bound the tails of $\sum_{i=1}^m \rv{L}_i$.
Based on the information provided by Lemma~\ref{lem:Labc} there are multiple ways to achieve this.
In this section we unfold a simple strategy based on Hoeffding's inequality that only uses points (\ref{it:expL}) and (\ref{it:rangeL}) above.
In Section~\ref{sec:improved_bounds} we discuss how to improve these bounds.
For now, the following result will suffice to obtain a privacy amplification bound for generic $\varepsilon_0$-LDP local randomizers.

\begin{lemma}\label{lem:Lhoeff}
Let $\rv{L}_1, \ldots, \rv{L}_m$ be i.i.d.\ bounded random variables with $\Ex\rv{L}_i = -a \leq 0$. Suppose $b_- \leq \rv{L}_i \leq b_+$ and let $b = b_+ - b_-$. Then the following holds:
\begin{align*}
\Ex\left[\sum_{i=1}^m \rv{L}_i\right]_+ \leq
\frac{b^2}{4 a} e^{- \frac{2 m a^2}{b^2}} \enspace.
\end{align*}
\end{lemma}

By combining Lemmas~\ref{lem:sumL},~\ref{lem:Labc} and~\ref{lem:Lhoeff} we immediately obtain the main theorem of this section.

\begin{theorem}\label{thm:bound_Hoeff_generic}
Let $\mech{R} : \dom{X} \to \dom{Y}$ be an $\varepsilon_0$-LDP local randomizer and let $\mech{M} = \mech{S} \circ \mech{R}^n$ be the corresponding shuffled mechanism.
Then $\mech{M}$ is $(\varepsilon,\delta)$-DP for any $\varepsilon$ and $\delta$ satisfying
\begin{align}\label{eqn:technical_bound}
\frac{(e^{\varepsilon} + 1)^2 (e^{\varepsilon_0} - e^{-\varepsilon_0})^2}{4 n (e^{\varepsilon}-1)}
e^{-C n \left(\frac{1}{e^{\varepsilon_0}} \wedge \frac{(e^{\varepsilon}-1)^2}{ (e^{\varepsilon}+1)^2 (e^{\varepsilon_0} - e^{-\varepsilon_0})^2}\right)} \leq \delta \enspace,
\end{align}
where $C = 1 - e^{-2} \approx 0.86$.
\end{theorem}

While it is easy to numerically test or solve \eqref{eqn:technical_bound}, extracting manageable asymptotics from this bound is less straightforward.
The following corollary massages this expression to distill insights about privacy amplification by shuffling for generic $\varepsilon_0$-LDP local randomizers.

\begin{corollary}\label{cor:bound_Hoeff_generic}
Let $\mech{R} : \dom{X} \to \dom{Y}$ be an $\varepsilon_0$-LDP local randomizer and let $\mech{M} = \mech{S} \circ \mech{R}^n$ be the corresponding shuffled mechanism.
If $\varepsilon_0 \leq \log(n / \log(1/\delta)) / 2$, then $\mech{M}$ is $(\varepsilon,\delta)$-DP with $\varepsilon = O((1 \wedge \varepsilon_0) e^{\varepsilon_0} \sqrt{\log(1/\delta) / n})$.
\end{corollary}

\subsection{Improved Amplification Bounds}\label{sec:improved_bounds}

There are at least two ways in which we can improve upon the privacy amplification bound in Theorem~\ref{thm:bound_Hoeff_generic}.
One is to leverage the moment information about the privacy amplification random variables provided by point (\ref{it:varL}) in Lemma~\ref{lem:Labc}.
The other is to compute more precise information about the privacy amplification random variables for specific mechanisms instead of using the generic bounds provided by Lemma~\ref{lem:Labc}.
In this section we give the necessary tools to obtain these improvements, which we then evaluate numerically in Section~\ref{sec:experiments}.

Hoeffding's inequality provides concentration for sums of bounded random variables.
As such, it is easy to apply because it requires little information on the behavior of the individual random variables.
On the other hand, this simplicity can sometimes provide sub-optimal results, especially when the random variables being added have standard deviation which is smaller than their range.
In this case one can obtain better results by applying one of the many concentration inequalities that take the variance of the summands into account.
The following lemma takes this approach by applying Bennett's inequality to bound the quantity $\Ex[\sum_{i=1}^m \rv{L}_i]_+$.

\begin{lemma}\label{lem:Lbenn}
Let $\rv{L}_1, \ldots, \rv{L}_m$ be i.i.d.\ bounded random variables with $\Ex\rv{L}_i = -a \leq 0$. Suppose $\rv{L}_i \leq b_+$ and $\Ex\rv{L}_i^2 \leq c$. Then the following holds:
\begin{align*}
\Ex\left[\sum_{i=1}^m \rv{L}_i\right]_+ \leq
\frac{b_+}{a m \log\left(1 + \frac{a b_+}{c}\right)} e^{- \frac{m c}{b_+^2} \phi\left(\frac{a b_+}{c}\right)}
\enspace,
\end{align*}
where $\phi(u) = (1+u) \log(1 + u) - u$.
\end{lemma}

This results can be combined with Lemmas~\ref{lem:gamma_generic}, \ref{lem:sumL} and \ref{lem:Labc} to obtain an alternative privacy amplification bound for generic $\varepsilon_0$-LDP local randomizers to the one provided in Theorem~\ref{thm:bound_Hoeff_generic}.
However, the resulting bound is cumbersome and does not have a nice closed-form like the one in Theorem~\ref{thm:bound_Hoeff_generic}.
Thus, instead of stating the bound explicitly we will evaluate it numerically in the following section.

The other way in which we can provide better privacy bounds is by making them mechanism specific.
Lemma~\ref{lem:gamma_specific} already gives exact expression for the total variation similarity $\gamma$ of three local randomizers.
To be able to apply Hoeffding's (Lemma~\ref{lem:Lhoeff}) and Bennett's (Lemma~\ref{lem:Lbenn}) inequalities to these local randomizers we need information about the range and the second moment of the corresponding privacy amplification random variables.
The following results provide this type of information for randomized response and the Laplace mechanism.

\begin{lemma}\label{lem:Labc_RR}
Let $\mech{R} : [k] \to [k]$ be the $k$-ary $\varepsilon_0$-LDP randomized response mechanism.
Let $\gamma = k / (e^{\varepsilon_0} + k -1)$ be the total variation similarity of $\mech{R}$ (cf.\ Lemma~\ref{lem:gamma_specific}).
For any $\varepsilon \geq 0$ and $x, x' \in \dom{X}$, $x \neq x'$, the privacy amplification random variable $\rv{L} = \rv{L}_\varepsilon^{x,x'}$ satisfies:
\begin{enumerate}
\item $-(1 - \gamma) k e^{\varepsilon} \leq \rv{L} - \gamma (1-e^{\varepsilon}) \leq (1-\gamma) k$,\label{it:range_l_rr}
\item $\Ex\rv{L}^2 = \gamma (2 - \gamma) (1 - e^{\varepsilon})^2 + (1-\gamma)^2 k (1 + e^{2 \varepsilon})$.\label{it:var_l_rr}
\end{enumerate}
\end{lemma}

\begin{lemma}\label{lem:Labc_Laplace}
Let $\mech{R} : [0,1] \to \R$ be the $\varepsilon_0$-LDP Laplace mechanism $\mech{R}(x) = x + \Laplace(1/\varepsilon_0)$.
For any $\varepsilon \geq 0$ and $x, x' \in \dom{X}$ the privacy amplification random variable $\rv{L} = \rv{L}_\varepsilon^{x,x'}$ satisfies:
\begin{enumerate}
\item $e^{-\varepsilon_0/2} (1 - e^{\varepsilon+\varepsilon_0}) \leq \rv{L} \leq e^{\varepsilon_0/2} (1 - e^{\varepsilon-\varepsilon_0})$,\label{it:range_l_lap}
\item $\Ex\rv{L}^2 \leq \frac{e^{2\varepsilon} + 1}{3} (2 e^{\varepsilon_0/2} + e^{-\varepsilon_0})
- 2 e^{\varepsilon} (2 e^{-\varepsilon_0/2} - e^{-\varepsilon_0})$.\label{it:var_l_lap}
\end{enumerate}
\end{lemma}

Again, instead of deriving a closed-form expression like \eqref{eqn:technical_bound} specialized to these two mechanisms, we will numerically evaluate the advantage of using mechanism-specific information in the bounds in the next section.
Note that we did not provide a version of these results for the Gaussian mechanism for which we showed how to compute $\gamma$ in Section~\ref{sec:blanket}.
The reason for this is that in this case the resulting privacy amplification random variables are not bounded.
This precludes us from using the Hoeffding and Bennett bounds to analyze the privacy amplification in this case. Approaches using concentration bounds that do not rely on boundedness will be explored in future work.

\section{Experimental Evaluation}
\label{sec:experiments}
\begin{figure}
\centering
\includegraphics[height=.45\textwidth]{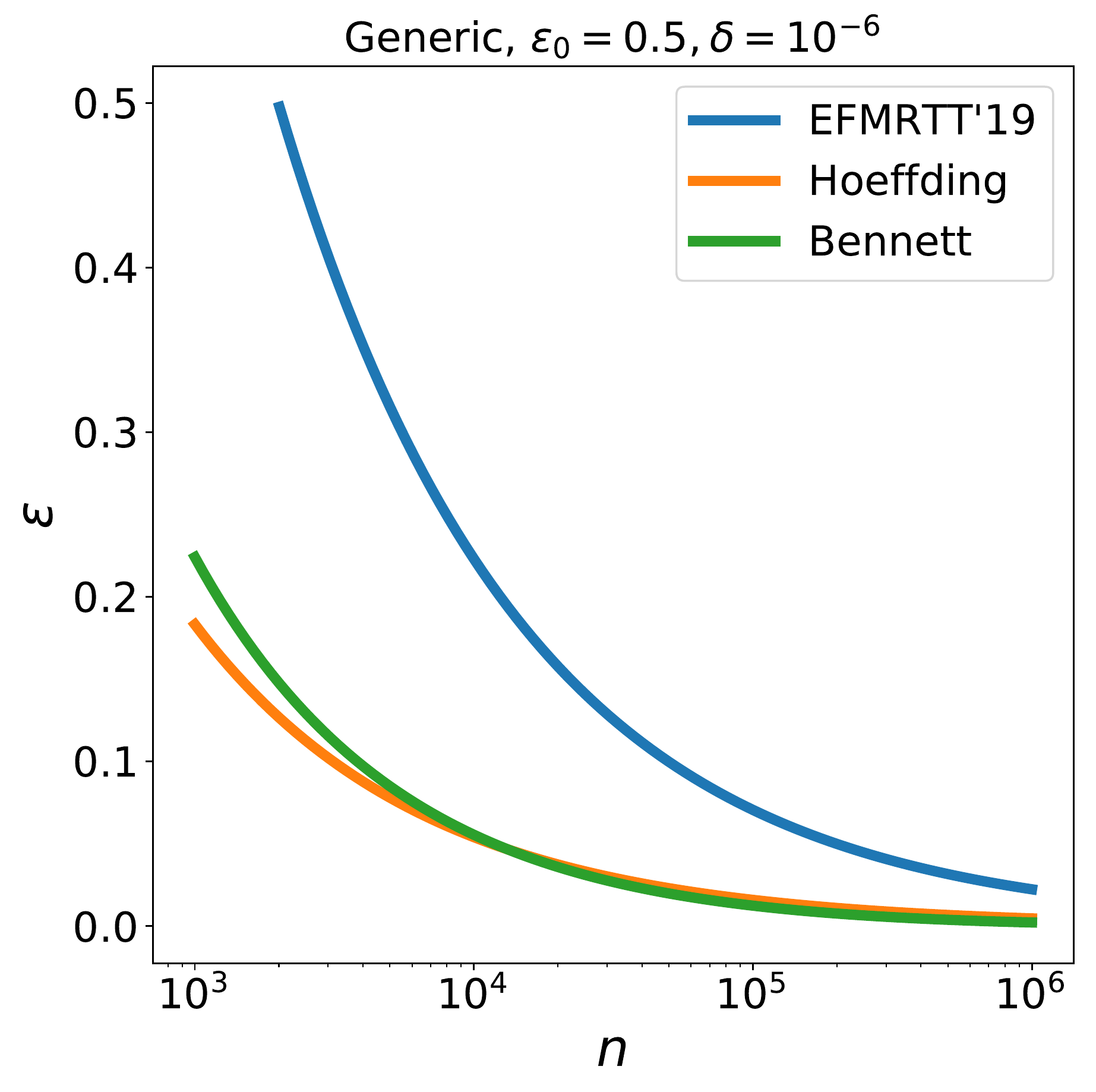}
\includegraphics[height=.45\textwidth]{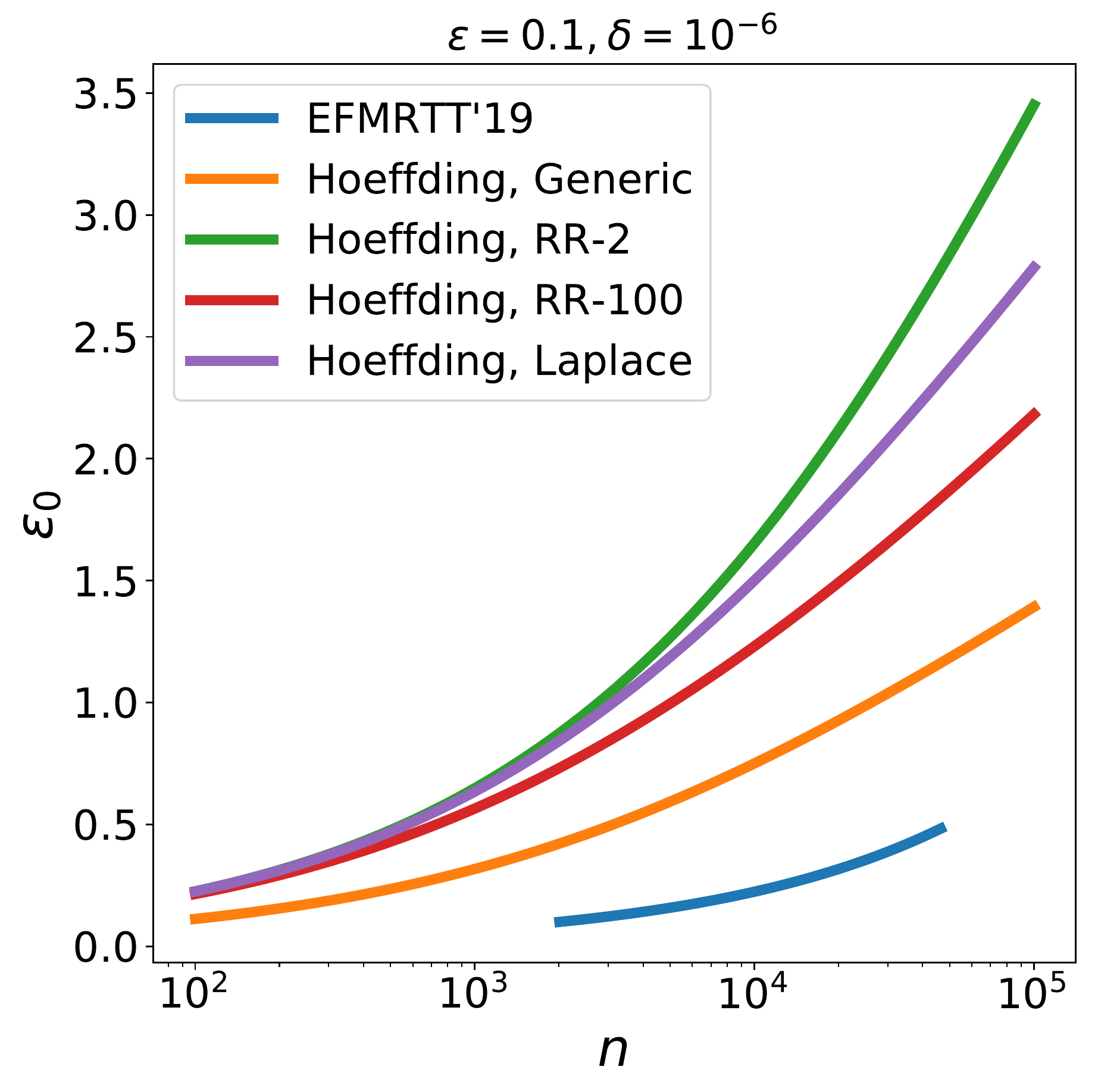}
\includegraphics[height=.45\textwidth]{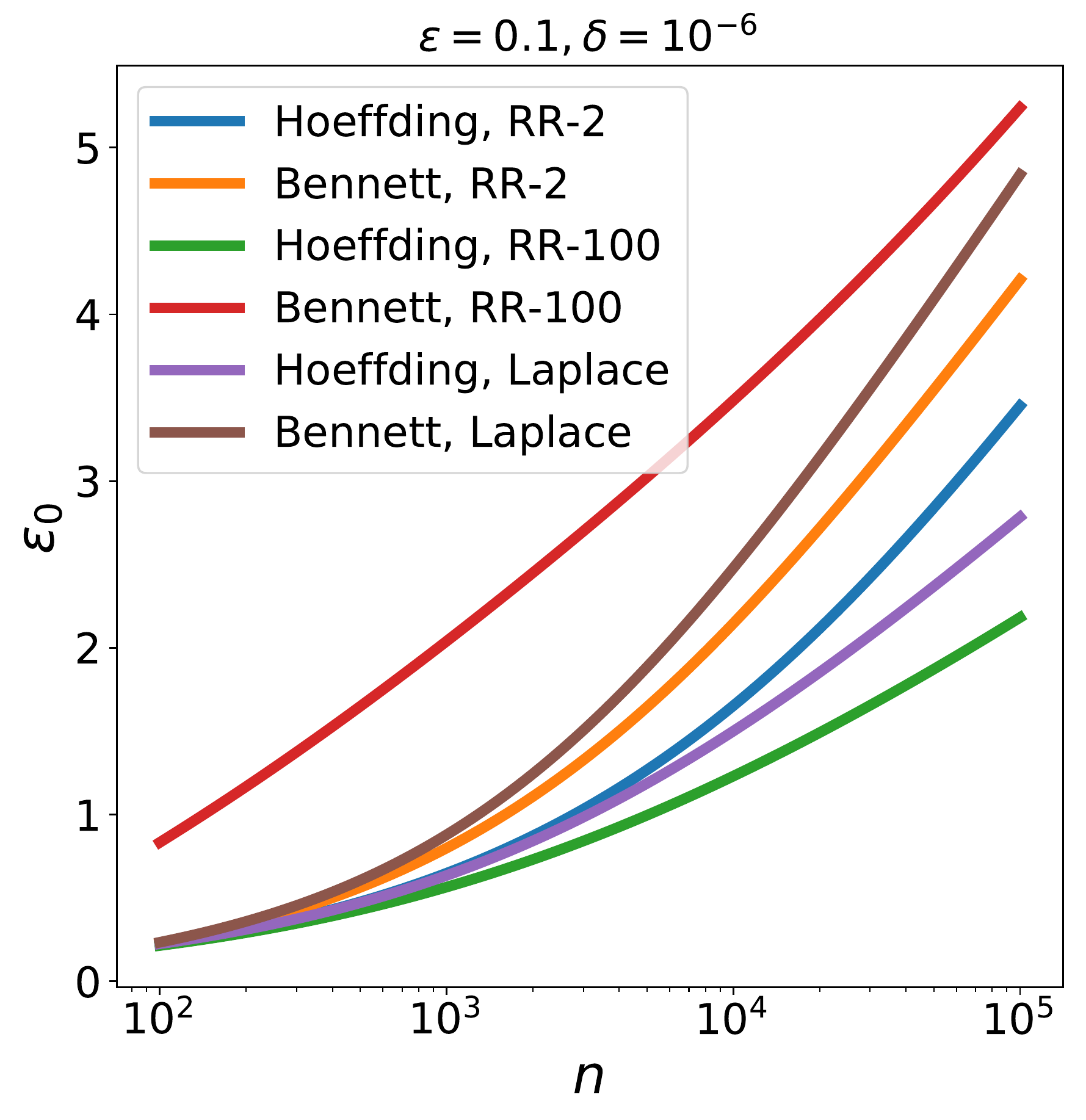}
\includegraphics[height=.45\textwidth]{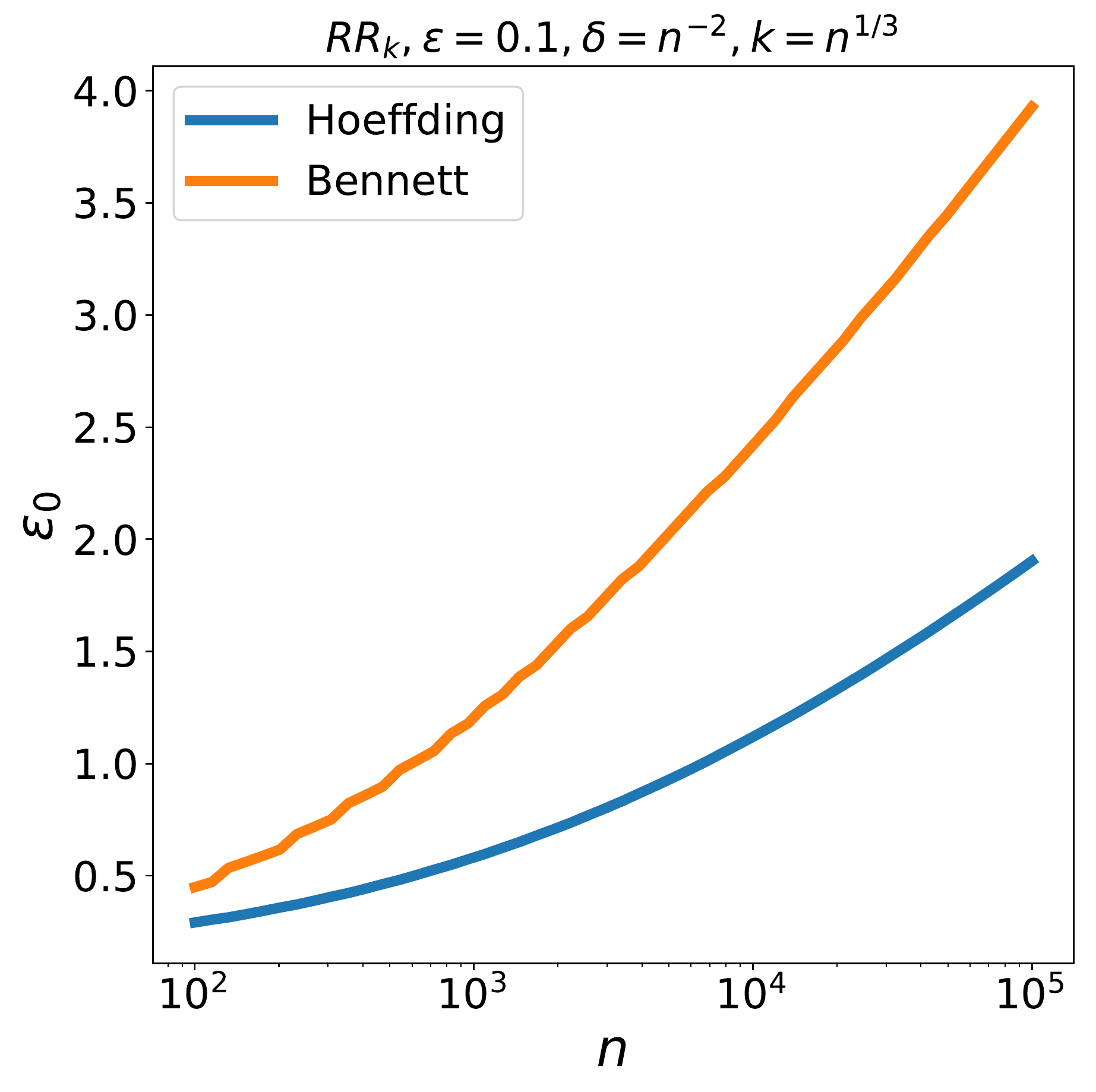}
\caption{(i) Comparison of $\varepsilon(n)$ for fixed $\varepsilon_0$ and $\delta$ of the bounds obtained for generic $\varepsilon_0$-DP local randomizers using the bound in \cite{erlingsson2019amplification} and our Hoeffding and Bennett bounds. (ii) Comparison of $\varepsilon_0(n)$ for fixed $\varepsilon$ and $\delta$ for generic and specific local randomizers using the Hoeffding bounds. (iii) Comparison of $\varepsilon_0(n)$ for fixed $\varepsilon$ and $\delta$ for specific local randomizers using the Hoeffding and Bennett bound. (iv) Comparison of $\varepsilon_0(n)$ for fixed $\varepsilon$ and $\delta = n^{-2}$ for a randomized response mechanism with domain size $k = n^{1/3}$ using the Hoeffding and Bennett bounds.}
\label{fig:experiments}
\end{figure}

In this section we provide a numerical evaluation of the privacy amplification bounds derived in Section~\ref{sec:amplification}.
We also compare the results obtained with our techniques to the privacy amplification bound of Erlingsson et al.\ \cite{erlingsson2019amplification}.

To obtain values of $\varepsilon$ and $\varepsilon_0$ from bounds on $\delta$ of the form given in Theorem~\ref{thm:bound_Hoeff_generic} we use a numeric procedure.
In particular, we implemented the bounds for $\delta$ in Python and then used SciPy's numeric root finding routines to solve for the desired parameter up to a precision of $10^{-12}$.
This leads to a simple and efficient implementation which can be employed in practical applications for the calibration of privacy parameters of local randomizers in shuffled protocols.
The resulting code is available at \url{https://github.com/BorjaBalle/amplification-by-shuffling}.

The results of our evaluation are given in Figure~\ref{fig:experiments}.
The bounds plotted in this figure are obtained as follows:
\begin{enumerate}
\item (EFMRTT'19) is the bound in \cite{erlingsson2019amplification} (see Theorem~\ref{thm:efmrtt}).
\item (Hoeffding, Generic) is the bound from Theorem~\ref{thm:bound_Hoeff_generic}.
\item (Bennett, Generic) is obtained by combining Lemmas~\ref{lem:gamma_generic}, \ref{lem:sumL}, \ref{lem:Labc} and \ref{lem:Lbenn}.
\item (Hoeffding, RR) is obtained by combining Lemmas~\ref{lem:gamma_specific}, \ref{lem:sumL}, \ref{lem:Labc_RR} and \ref{lem:Lhoeff}.
\item (Bennett, RR) is obtained by combining Lemmas~\ref{lem:gamma_specific}, \ref{lem:sumL}, \ref{lem:Labc_RR} and \ref{lem:Lbenn}.
\item (Hoeffding, Laplace) is  obtained by combining Lemmas~\ref{lem:gamma_specific}, \ref{lem:sumL}, \ref{lem:Labc_Laplace} and \ref{lem:Lhoeff}.
\item (Bennett, Laplace) is obtained by combining Lemmas~\ref{lem:gamma_specific}, \ref{lem:sumL}, \ref{lem:Labc_Laplace} and \ref{lem:Lbenn}.
\end{enumerate} 

In panel (i) we observe that our two bounds for generic randomizers give significantly smaller values of $\varepsilon$ than the bound from \cite{erlingsson2019amplification} where the constants where not optimized.
Additionally, we see that for generic local randomizers, Hoeffding is better for small values of $n$, while Bennet is better for large values of $n$.
In panel (ii) we observe the advantage of incorporating information in the Hoeffding bound about the specific local randomizer.
Additionally, this plot allows us to see that for the same level of local DP, binary randomized response has better amplification properties than Laplace, which in turn is better the randomizer response over a domain of size $k = 100$.
In panel (iii) we compare the amplification bounds obtained for specific randomizers with the Hoeffding and Bennett bounds.
We observe that for every mechanism the Bennett bound is better than the Hoeffding bound, especially for large values of $n$.
Additionally, the gain of using Bennett instead of Hoeffding is greater for randomized response with $k = 100$ than for other mechanisms. The reason for this is that for fixed $\varepsilon_0$ and large $k$, the total variation similarity of randomized response is close to $1$ (cf.\ Lemma~\ref{lem:gamma_specific}).
Finally, in panel (iv) we compare the values of $\varepsilon_0$ obtained for a randomized response with domain size growing with the number of users as $k = n^{1/3}$.
This is in line with our optimal protocol for real summation in the single-message shuffle model presented in Section~\ref{sec:optimal_summation}.
We observe that also in this case the Bennett bounds provides a significant advantage over Hoeffding.

To summarize, we showed that our generic bounds outperform the previous amplification bounds developed in \cite{erlingsson2019amplification}.
Additionally, we showed that incorporating both information about the variance of the privacy amplification random variable via the use of Bennett's bound, as well as information about the behavior of this random variable for specific mechanisms, leads to significant improvements in the privacy parameters obtained for shuffled protocols.
This is important in practice because being able to maximize the $\varepsilon_0$ parameter for the local randomizer -- while satisfying a prescribed level of differential privacy in the shuffled protocol -- leads to more accurate protocols.
 
\section{Conclusion}
\label{sec:conclusion}

We have shown a separation result for the single-message shuffle model, showing that it can not achieve the level of accuracy of the curator model of differential privacy, but that it can yield protocols that are significantly more accurate than the ones from the local model. More specifically, we provided a single message protocol for private $n$-party summation of real values in $[0,1]$ with $O(\log n)$-bit communication and $O(n^{1/6})$ standard deviation. We also showed that our protocol is optimal in terms of accuracy by providing a lower bound for this problem. In previous work, Cheu et al.~\cite{DBLP:journals/corr/abs-1808-01394} had shown that the selection problem can be solved more accurately in the central model than in the shuffle model, and that the real summation problem can be solved more accurately in the shuffle model than in the local model. For the former, they rely on lower bounds for selection in the local model by means of a generic reduction from the shuffle to the local model, while our lower bound is directly in the shuffle model, offering additional insight. On the other hand, our single-message protocol for summation is more accurate than theirs.

Moreover, we introduced the notion of the privacy blanket of a local randomizer, and show how it allows us to give a generic treatment to the problem of obtaining privacy amplification bounds in the shuffle model that improves on recent work by Erlingsson et al.~\cite{erlingsson2019amplification} and Cheu et al.~\cite{DBLP:journals/corr/abs-1808-01394}. Crucially, unlike the proofs in~\cite{erlingsson2019amplification,DBLP:journals/corr/abs-1808-01394}, our proof does not rely on privacy amplification by subsampling. We believe that the notion of the privacy blanket is of interest beyond the shuffle model, as it leads to a canonical decomposition of local randomizers that might be useful also in the study of the local model of differential privacy.
For example, Joseph et al.~\cite{DBLP:journals/corr/abs-1904-03564} already used a generalization of our blanket decomposition in their study of the role of interactivity in local DP protocols.

\bibliographystyle{plain}
\bibliography{../draft.bib}

\appendix

\section{Proofs}
\label{sec:proofs}

\subsection{Proofs from Section~\ref{sec:blanket}}

\begin{proof}[Proof of Lemma~\ref{lem:gamma_specific}]
To obtain (\ref{it:gamma_rr}) recall that an $\varepsilon_0$-LDP randomized response mechanism $\mech{R}$ over $[k]$ satisfies
\begin{align*}
\Pr[\mech{R}(x) = x] &= \frac{e^{\varepsilon_0}}{e^{\varepsilon_0} + k - 1} \enspace, \\
\Pr[\mech{R}(x) = x'] &= \frac{1}{e^{\varepsilon_0} + k - 1} \enspace,
\end{align*}
for $x' \neq x$. Therefore, we get
\begin{align*}
\gamma_{\mech{R}} = \sum_{y \in [k]} \min_{x \in [k]} \Pr[\mech{R}(x) = y]
=
\frac{k}{e^{\varepsilon_0} + k - 1} \enspace.
\end{align*}

To obtain (\ref{it:gamma_lap}) recall that an $\varepsilon_0$-LDP Laplace mechanism $\mech{R} : [0,1] \to \R$ has distribution $\mu_x(y) = \frac{\varepsilon_0}{2} e^{-\varepsilon_0 |y-x|}$.
Thus, for any $y \in \R$ we have
\begin{align*}
\inf_{x \in [0,1]} \mu_x(y) = \frac{\varepsilon_0}{2} \min\{e^{-\varepsilon_0 |y|}, e^{-\varepsilon_0 |y-1|} \} \enspace.
\end{align*}
We can use to decompose the definition of $\gamma_{\mech{R}}$ into the sum of two integrals as follows:
\begin{align*}
\gamma_{\mech{R}}
=
\int_{-\infty}^{\infty} \inf_{x \in [0,1]} \mu_x(y)
=
\frac{\varepsilon_0}{2}
\left(\int_{-\infty}^{\frac{1}{2}} e^{-\varepsilon_0 |y-1|} + \int_{\frac{1}{2}}^{\infty} e^{-\varepsilon_0 |y|} \right) \enspace.
\end{align*}
Performing the change of variables $z = y - 1/2$ in the first integral yields
\begin{align*}
\frac{\varepsilon_0}{2} \int_{-\infty}^{\frac{1}{2}} e^{-\varepsilon_0 |y-1|}
=
\frac{\varepsilon_0}{2} \int_{-\infty}^{0} e^{-\varepsilon_0 |z| - \frac{\varepsilon_0}{2}}
=
\frac{e^{-\varepsilon_0 / 2}}{2} \enspace.
\end{align*}
Similarly, for the second integral we also have
\begin{align*}
\frac{\varepsilon_0}{2} \int_{\frac{1}{2}}^{\infty} e^{-\varepsilon_0 |y|}
=
\frac{e^{-\varepsilon_0 / 2}}{2} \enspace.
\end{align*}
Thus, $\gamma_{\mech{R}} = e^{-\varepsilon_0/2}$.
We note for future reference that this argument also shows that the blanket distribution of a Laplace mechanism is again a Laplace distribution. In particular, we have $\omega_{\mech{R}}(y) = \frac{\varepsilon_0}{2} e^{-\varepsilon_0 |y - 1/2|}$.

To obtain (\ref{it:gamma_gauss}) recall that a Gaussian local randomizer $\mech{R} : [0,1] \to \R$ with variance $\sigma^2$ has distribution $\mu_x(y) = e^{-(y-x)^2 / 2 \sigma^2} / \sqrt{2 \pi \sigma^2}$.
Therefore, for any $y \in \R$ we have
\begin{align*}
\inf_{x \in [0,1]} \mu_x(y) =
\frac{1}{\sqrt{2 \pi \sigma^2}} \min\{ e^{-y^2 / 2 \sigma^2}, e^{-(y-1)^2 / 2 \sigma^2} \} \enspace.
\end{align*}
Integrating this expression over $y \in \R$ we get
\begin{align*}
\gamma_{\mech{R}}
=
\Pr[\Normal(1,\sigma^2) \leq 1/2]
+
\Pr[\Normal(0,\sigma^2) \geq 1/2]
=
2 \Pr[\Normal(0,\sigma^2) \geq -1/2] \enspace,
\end{align*}
where we used the symmetry of the Gaussian distribution around its mean.
\end{proof}

\begin{proof}[Proof of Lemma~\ref{lem:gamma_generic}]
Fix an arbitrary $x_0 \in \dom{X}$. Expanding the definition of total variation similarity and using the $\mech{R}$ is $\varepsilon_0$-LDP we get
\begin{align*}
\gamma = \int \inf_x \mu_x(y) = \int \frac{\inf_x \mu_x(y)}{\mu_{x_0}(y)} \mu_{x_0}(y)
\geq e^{-\varepsilon_0} \enspace.
\end{align*}
\end{proof}

\subsection{Proof of Lemma~\ref{lem:sumL}}

The proof of Lemma~\ref{lem:sumL} requires a number of intermediate steps we formalize as lemmas.
Before stating and proving these lemmas we need to introduce some notation.

Let $\mech{R} : \dom{X} \to \dom{Y}$ be a local randomizer with total variation similarity $\gamma$ and blanket distribution $\omega$.
For $x \in \dom{X}$ we write $\mu_x$ for the distribution of $\mech{R}(x)$ and recall that we have the mixture decompositions $\mu_x = (1-\gamma) \nu_x + \gamma \omega$.

Let $\mech{M} = \mech{S} \circ \mech{R}^n$ be the shuffling of $\mech{R}$.
Fixing an input $\tup{x} \in \dom{X}^n$ we define the random variables $\rv{Y}_i \sim \mu_{x_i}$ for $i \in [n]$.
Now we can consider the output of $\mech{M}(\tup{x})$ as a realization of the \emph{random multiset} $\rvset{Y} = \{\rv{Y}_1,\ldots,\rv{Y}_n\} \in \N_n^{\dom{Y}}$, where $\N_n^{\dom{Y}}$ denotes the collection of all multisets of cardinality $n$ with elements in $\dom{X}$.
Similarly, for $\tup{x}' \in \dom{X}^n$ with $\tup{x} \simeq \tup{x}'$, $x_n \neq x'_n$, we define the output of $\mech{M}(\tup{x}')$ as a realization of the random multiset $\rvset{Y}' = \{\rv{Y}_1, \ldots, \rv{Y}_{n-1}, \rv{Y}'_n\}$.
Thus, our goal is to bound $\div_{e^\varepsilon}(\rvset{Y} \| \rvset{Y}')$, where we slightly abuse our divergence notation by applying it to random variables instead of distributions.

In order to exploit the mixture decomposition provided by the blanket of $\mech{R}$ we define additional random variables.
Let $\rv{V}_i \sim \nu_{x_i}$ for $i \in [n-1]$ and let $\rv{W}_1, \ldots, \rv{W}_{n-1}$ be i.i.d.\ random variables with $\rv{W}_i \sim \omega$.
Thus, for $i \in [n-1]$ we have
\begin{align*}
\rv{Y}_i =
\begin{cases}
\rv{V}_i & \text{with probability $\gamma$} \enspace,\\
\rv{W}_i & \text{with probability $1 - \gamma$} \enspace.
\end{cases}
\end{align*}

Finally, we define $\rvset{B} \subseteq [n-1]$ to be the random subset of users among the first $n - 1$ who sample from the blanket, and let $\bar{\rvset{B}} = [n-1] \setminus \rvset{B}$.
Note that for any $B \subseteq [n-1]$ we have $\Pr[\rvset{B} = B] = \gamma^{|B|} (1-\gamma)^{n-1-|B|}$.
Conditioned on a particular value for the set of users who sample from the blanket we have
\begin{align*}
\rvset{Y} | \{\rvset{B} = B\} = \rvset{W}_B \cup \rvset{V}_{\bar{B}} \cup \{\rv{Y}_n\} \enspace,
\end{align*}
where $\rvset{W}_B = \{ \rv{W}_i \; | \; i \in B\}$ and $\rvset{V}_{\bar{B}} = \{ \rv{V}_i \; | \; i \in [n-1] \setminus B\}$.

With the notation defined above we can now state the following result, which shows that to bound $\div_{e^\varepsilon}(\rvset{Y} \| \rvset{Y}')$ it is enough to bound the divergences between the conditional random variables $\rvset{Y} | \{\rvset{B} = B\}$ and $\rvset{Y}' | \{\rvset{B} = B\}$ for all possible choices of the set $B \subseteq [n-1]$ of users who sample from the blanket.

\begin{lemma}\label{lem:extractB}
Fix $\varepsilon \geq 0$. Given $B \subseteq [n-1]$ let
\begin{align*}
\div_B = \div_{e^\varepsilon}(\rvset{W}_B \cup \rvset{V}_{\bar{B}} \cup \{\rv{Y}_n\} \| \rvset{W}_B \cup \rvset{V}_{\bar{B}} \cup \{\rv{Y}'_n\}) \enspace.
\end{align*}
Then the following holds:
\begin{align}\label{eqn:extractB}
\div_{e^\varepsilon}(\rvset{Y} \| \rvset{Y}')
\leq
\sum_{B \subseteq [n-1]} \gamma^{|B|} (1-\gamma)^{n-1-|B|} \div_B \enspace.
\end{align}
\end{lemma}
\begin{proof}
Recall that the hockey-stick divergence $\div_{e^{\varepsilon}}$ is an $f$-divergence in the sense of Csisz{\'a}r; this can be seen by taking $f(u) = [u - e^{\varepsilon}]_+$.
The result follows from a standard application of the joint convexity property of $f$-divergences.
\end{proof}

The next step in the proof is to ignore the contribution of any user among the first $n-1$ who do not sample from the blanket. In mathematical terms, and using the notation from Lemma~\ref{lem:extractB}, this is stated as
\begin{align}\label{eqn:ignoreBbar}
\div_B
\leq
\div_{e^\varepsilon}(\rvset{W}_B \cup \{\rv{Y}_n\} \| \rvset{W}_B \cup \{\rv{Y}'_n\}) \enspace.
\end{align}
To obtain such inequality we use the following lemma.

\begin{lemma}\label{lem:ignoreBbar}
Let $\rvset{A}_0, \rvset{A}, \rvset{A}'$ be random multisets of fixed cardinality with $|\rvset{A}| = |\rvset{A}'|$. Then the following holds:
\begin{align*}
\div_{e^\varepsilon}( \rvset{A}_0 \cup \rvset{A} \|  \rvset{A}_0 \cup \rvset{A}')
\leq
\div_{e^\varepsilon}( \rvset{A} \|  \rvset{A}') \enspace.
\end{align*}
\end{lemma}
\begin{proof}
We shall prove the result for $|\rvset{A}_0| = 1$. The general result follows directly by induction on the size of $\rvset{A}_0$.

Suppose $|\rvset{A}_0| + |\rvset{A}| = m$ and $\rvset{A}_0 = \{\rv{A}\}$ for some random variable $\rv{A}$.
For any multiset $Y \in \N_m^{\dom{Y}}$ we can write
\begin{align*}
\Pr[\rvset{A}_0 \cup \rvset{A} = Y]
=
\sum_{y \in \dom{Y}} \Pr[\rv{A} = y] \Pr[\rvset{A} = Y \setminus \{y\}] \enspace,
\end{align*}
where we take the convention that $\Pr[\rvset{A} = Y \setminus \{y\}] = 0$ whenever $y \notin Y$.
Now we expand the definition of $\div_{e^\varepsilon}$ to get:
\begin{align*}
&\div_{e^\varepsilon}( \rvset{A}_0 \cup \rvset{A} \|  \rvset{A}_0 \cup \rvset{A}')
=
\int_{\N_m^{\dom{Y}}} \left[\Pr[\rvset{A}_0 \cup \rvset{A} = Y] - e^{\varepsilon} \Pr[\rvset{A}_0 \cup \rvset{A}' = Y]\right]_+
\\
&\qquad \leq
\int_{\N_m^{\dom{Y}}} \sum_{y \in \dom{Y}} \Pr[\rv{A} = y] \left[\Pr[\rvset{A} = Y \setminus \{y\}] - e^{\varepsilon} \Pr[\rvset{A}' = Y \setminus \{y\}]\right]_+
\\
&\qquad =
\int_{\N_{m-1}^{\dom{Y}}} \left[\Pr[\rvset{A} = Y] - e^{\varepsilon} \Pr[\rvset{A}' = Y]\right]_+
=
\div_{e^\varepsilon}( \rvset{A} \|  \rvset{A}') \enspace.
\end{align*}
\end{proof}

Taking $\rvset{A}_0 = \rvset{V}_{\bar{B}}$ in Lemma~\ref{lem:ignoreBbar} yields \eqref{eqn:ignoreBbar}.
Now we observe that since the random variables $\rv{W}_i$, $i \in [n-1]$, are i.i.d., the distribution of the random multiset $\rvset{W}_B$ only depends on $B$ through its cardinality $m = |B|$.
Accordingly, we define $\rvset{W}_m = \{\rv{W}_1, \ldots, \rv{W}_m\}$ for $m \in \{0,1,\ldots,n-1\}$, where $\rvset{W}_0 = \emptyset$.
This allows us to summarize the argument so far as showing that $\div_{e^\varepsilon}(\rvset{Y} \| \rvset{Y}')$ can be upper bounded by
\begin{align*}
\sum_{m=0}^{n-1} \binom{n-1}{m} \gamma^{m} (1-\gamma)^{n-1-m} \div_{e^{\varepsilon}}(\rvset{W}_m \cup \{\rv{Y}_n\} \| \rvset{W}_m \cup \{\rv{Y}'_n\}) \enspace.
\end{align*}
The next step in the proof is to obtain an expression for the divergences in this expression in terms of the privacy amplification random variables.
This is done in the following lemma.

\begin{lemma}
For any $m \geq 1$ we have
\begin{align*}
\div_{e^{\varepsilon}}(\rvset{W}_{m-1} \cup \{\rv{Y}_n\} \| \rvset{W}_{m-1} \cup \{\rv{Y}'_n\})
=
\Ex\left[\frac{1}{m} \sum_{i=1}^{m} \rv{L}_i\right]_+ \enspace.
\end{align*}
\end{lemma}
\begin{proof}
Let $\tup{y} \in \dom{Y}^m$ be a tuple of elements from $\dom{Y}$ and $Y \in \N_{m}^{\dom{Y}}$ be the corresponding multiset of entries.
Then we have
\begin{align*}
\Pr[\rvset{W}_{m-1} \cup \{\rv{Y}_n\} = Y]
&=
\frac{1}{m!} \sum_{\sigma} \Pr[(\rv{W}_{1}, \ldots, \rv{W}_{m-1}, \rv{Y}_{n}) = \tup{y}_{\sigma}]
\enspace,
\end{align*}
where $\sigma$ ranges over all permutations of $[m]$ and we write $\tup{y}_{\sigma} = (y_{\sigma(1)}, \ldots, y_{\sigma(m)})$.
Now note that since $\rv{W}_i \sim \omega$ and $\rv{Y}_n \sim \mu_{x_n}$, we also have
\begin{align*}
\Pr[(\rv{W}_{1}, \ldots, \rv{W}_{m-1}, \rv{Y}_{n}) = \tup{y}_{\sigma}]
= \omega(y_{\sigma(1)}) \cdots \omega(y_{\sigma(m-1)}) \mu_{x_n}(y_{\sigma(m)})
\enspace.
\end{align*}
Summing this expression over all permutations $\sigma$ and factoring out the product of the $\omega$'s yields:
\begin{align*}
&\frac{1}{m!} \sum_{\sigma} \omega(y_{\sigma(1)}) \cdots \omega(y_{\sigma(m-1)}) \mu_{x_n}(y_{\sigma(m)})
\\
&\qquad\qquad=
\left(\prod_{i=1}^m \omega(y_i)\right)
\frac{1}{m} \sum_{i=1}^m \frac{\mu_{x_n}(y_i)}{\omega(y_i)}
\\
&\qquad\qquad=
\Pr[\rvset{W}_m = Y] \frac{1}{m} \sum_{i=1}^m \frac{\mu_{x_n}(y_i)}{\omega(y_i)}
\enspace.
\end{align*}
Now we can plug these observation into the definition of $\div_{e^{\varepsilon}}$ and complete the proof as follows:
\begin{align*}
&\div_{e^{\varepsilon}}(\rvset{W}_{m-1} \cup \{\rv{Y}_n\} \| \rvset{W}_{m-1} \cup \{\rv{Y}'_n\})
\\
&\qquad =
\int_{\N_m^{\dom{Y}}} \left[\Pr[\rvset{W}_{m-1} \cup \{\rv{Y}_n\} = Y] - e^{\varepsilon} \Pr[\rvset{W}_{m-1} \cup \{\rv{Y}'_n\} = Y]\right]_+
\\
&\qquad =
\int_{\N_m^{\dom{Y}}} \Pr[\rvset{W}_m = Y] \left[\frac{1}{m} \sum_{i=1}^m \frac{\mu_{x_n}(y_i) - e^{\varepsilon}\mu_{x'_n}(y_i) }{\omega(y_i)}\right]_+
\\
&\qquad =
\Ex\left[\frac{1}{m} \sum_{i=1}^m \frac{\mu_{x_n}(\rv{W}_i) - e^{\varepsilon}\mu_{x'_n}(\rv{W}_i) }{\omega(\rv{W}_i)} \right]_+
\\
&\qquad =
\Ex\left[\frac{1}{m} \sum_{i=1}^{m} \rv{L}_i\right]_+ \enspace.
\end{align*}
\end{proof}

To conclude the proof of Lemma~\ref{lem:sumL} we perform a change of variable to obtain
\begin{align*}
\div_{e^\varepsilon}(\rvset{Y} \| \rvset{Y}')
&\leq
\sum_{m=0}^{n-1} \binom{n-1}{m} \gamma^{m} (1-\gamma)^{n-1-m} \Ex\left[\frac{1}{m+1} \sum_{i=1}^{m+1} \rv{L}_i\right]_+ \\
&=
\frac{1}{\gamma n} \sum_{m=1}^{n} \binom{n}{m} \gamma^{m} (1-\gamma)^{n-m} \Ex\left[\sum_{i=1}^{m} \rv{L}_i\right]_+ \enspace.
\end{align*}

We note that despite the length of the proof, only two inequalities were used to obtain the result.
The one in Lemma~\ref{lem:extractB} which follows from joint convexity, and the one in Lemma~\ref{lem:ignoreBbar} which is a post-processing type property.

\subsection{Other Proofs from Section~\ref{sec:amplification_bounds}}

\begin{proof}[Proof of Lemma~\ref{lem:Labc}]
Let $\rv{W} \sim \omega$.
Then, for any $x \in \dom{X}$ we have
\begin{align*}
\Ex\left[\frac{\mu_{x}(W)}{\omega(W)}\right]
=
\int \frac{\mu_{x}(y)}{\omega(y)} \omega(y) dy
=
\int \mu_{x}(y) dy
=
1 \enspace.
\end{align*}
Thus, the first claim follows by linearity of expectation:
\begin{align*}
\Ex\rv{L}
=
\Ex\left[\frac{\mu_{x}(\rv{W}) - e^{\varepsilon} \mu_{x'}(\rv{W})}{\omega(\rv{W})}\right]
=
1 - e^{\varepsilon} \enspace.
\end{align*}

For the second claim we expand the definition of $\omega$ to write
\begin{align*}
\frac{\mu_x(y) - e^{\varepsilon} \mu_{x'}(y)}{\omega(y)}
=
\gamma \frac{\mu_x(y) - e^{\varepsilon} \mu_{x'}(y)}{\inf_{x_0} \mu_{x_0}(y)}
\end{align*}
and then use that $\mech{R}$ is $\varepsilon_0$-LDP to get
\begin{align*}
e^{-\varepsilon_0} - e^{\varepsilon+\varepsilon_0}
\leq
\frac{\mu_x(y) - e^{\varepsilon} \mu_{x'}(y)}{\inf_{x_0} \mu_{x_0}(y)}
\leq e^{\varepsilon_0} - e^{\varepsilon-\varepsilon_0} \enspace.
\end{align*}
for any $y \in \dom{Y}$.

To prove the third claim we note that since $\mech{R}$ is $\varepsilon_0$-LDP we have
\begin{align*}
\Ex\left[\left(\frac{\mu_{x}(W)}{\omega(W)}\right)^2\right]
=
\int \left(\frac{\mu_{x}(y)}{\omega(y)}\right)^2 \omega(y) dy
=
\int \frac{\mu_{x}(y)}{\omega(y)} \mu_{x}(y) dy
\leq
\gamma e^{\varepsilon_0} \enspace.
\end{align*}
Furthermore, we can use a similar argument to show that
\begin{align*}
\Ex\left[\frac{\mu_{x}(W) \mu_{x'}(W)}{\omega(W)^2}\right]
=
\int \frac{\mu_{x}(y) \mu_{x'}(y)}{\omega(y)^2} \omega(y) dy
\geq
\gamma^2 e^{-2 \varepsilon_0}
\enspace.
\end{align*}
Plugging the last two bounds together we obtain
\begin{align*}
\Ex\rv{L}^2
&=
\Ex\left[\left(\frac{\mu_{x}(W) - e^{\varepsilon} \mu_{x'}(W)}{\omega(W)}\right)^2\right]
\leq
\gamma e^{\varepsilon_0} (e^{2 \varepsilon} + 1) - 2 \gamma^2 e^{\varepsilon - 2\varepsilon_0} \enspace.
\end{align*}
\end{proof}

\begin{lemma}\label{lem:intexp}
Suppose $h : [a, \infty) \to \R$ is a differentiable function such that $\lim_{t \to \infty} h(t) = \infty$ and $h'(t)$ is monotonically increasing. Then the following holds:
\begin{align*}
\int_{a}^{\infty} e^{-h(t)} \leq \frac{e^{-h(a)}}{h'(a)} \enspace.
\end{align*}
\end{lemma}
\begin{proof}
Note $\frac{d}{dt} e^{-h(t)} = -h'(t) e^{-h(t)}$. Thus, we can write
\begin{align*}
\int_{a}^{\infty} e^{-h(t)}
=
\int_{a}^{\infty} \frac{\frac{d}{dt} e^{-h(t)}}{-h'(t)}
\leq
\frac{-1}{h'(a)} \int_{a}^{\infty} \frac{d}{dt} e^{-h(t)}
=
\frac{e^{-h(a)}}{h'(a)} \enspace.
\end{align*}
\end{proof}

\begin{proof}[Proof of Lemma~\ref{lem:Lhoeff}]
Recall that for any non-negative random variable $\rv{L}$ we have $\Ex\rv{L} = \int_{0}^{\infty} \Pr[\rv{L} > t] dt$.
Furthermore, taking $\rv{L} = \sum_{i=1}^m \rv{L}_i$ we have $\Pr[[\rv{L}]_+ > t] = \Pr[\rv{L} > t]$ for any $t \geq 0$.
Under our assumptions on $\rv{L}_i$ we can use Hoeffding's inequality to show that
\begin{align*}
\Pr[\rv{L} > t]
=
\Pr[\rv{L} - \Ex\rv{L} > t + a m]
\leq
e^{-\frac{2 (t + a m)^2}{m b^2}} \enspace.
\end{align*}
Finally, applying Lemma~\ref{lem:intexp} with $h(t) = \frac{2 (t + a m)^2}{m b^2}$ we obtain
\begin{align*}
\Ex[\rv{L}]_+ \leq \int_0^{\infty} e^{-\frac{2 (t + a m)^2}{m b^2}} \leq \frac{b^2}{4 a} e^{- \frac{2 m a^2}{b^2}} \enspace.
\end{align*}
\end{proof}

\begin{proof}[Proof of Theorem~\ref{thm:bound_Hoeff_generic}]
Suppose $\mech{R}$ has total variation similarity $\gamma$.
By Lemma~\ref{lem:Labc} we can apply Lemma~\ref{lem:Lhoeff} to bound the expectations $\Ex[\sum_{i=1}^m \rv{L}_i]$ in Lemma~\ref{lem:sumL} with $a = e^{\varepsilon} - 1$ and $b = \gamma (e^{\varepsilon}+1) (e^{\varepsilon_0} - e^{-\varepsilon_0})$.
Thus, using the binomial identity
\begin{align*}
\sum_{m=0}^n \binom{n}{m} \gamma^m (1-\gamma)^{n-m} e^{-s m} = (1 - \gamma (1-e^{-s}))^n
\end{align*}
we get
\begin{align*}
\sum_{m=1}^n \binom{n}{m} \gamma^m (1-\gamma)^{n-m}
\Ex\left[\sum_{i=1}^m \rv{L}_i\right]_+
&\leq
\frac{b^2}{4 a} \left(1 - \gamma \left(1-e^{-\frac{2 a^2}{b^2}}\right)\right)^n \\
&\leq
\frac{b^2}{4 a} e^{-\gamma n \left(1-e^{-\frac{2 a^2}{b^2}}\right)} \enspace.
\end{align*}
Now we use $1 - e^{-2 x} \geq C (1 \wedge x)$ to see that
\begin{align*}
&\frac{1}{\gamma n} \frac{b^2}{4 a} e^{-\gamma n \left(1-e^{-\frac{2 a^2}{b^2}}\right)}
\leq
\frac{1}{\gamma n} \frac{b^2}{4 a} e^{-C \gamma n \left(1 \wedge \frac{a^2}{b^2}\right)} \\
&\qquad =
\frac{\gamma (e^{\varepsilon} + 1)^2 (e^{\varepsilon_0} - e^{-\varepsilon_0})^2}{4 n (e^{\varepsilon}-1)}
e^{-C n \left(\gamma \wedge \frac{(e^{\varepsilon}-1)^2}{\gamma (e^{\varepsilon}+1)^2 (e^{\varepsilon_0} - e^{-\varepsilon_0})^2}\right)}
\end{align*}
The bound follows from substituting the inequalities $e^{-\varepsilon_0} \leq \gamma \leq 1$ (Lemma~\ref{lem:gamma_generic}) above.
\end{proof}

\begin{proof}[Proof of Corollary~\ref{cor:bound_Hoeff_generic}]
To obtain the desired result we first massage the LHS of \eqref{eqn:technical_bound} and then solve for $\varepsilon$ in the resulting inequality.
We start by observing that $e^{\varepsilon_0} - e^{-\varepsilon_0} = O((1 \wedge \varepsilon_0) e^{\varepsilon_0})$.
Furthermore, since the assumption $\varepsilon_0 \leq \log(n / \log(1/\delta)) / 2$ implies $\varepsilon = O(1)$, we have $(e^{\varepsilon} - 1)/(e^{\varepsilon} + 1) = \Omega(\varepsilon)$.
Plugging these bounds in the exponential term on the LHS of \eqref{eqn:technical_bound} we see that
\begin{align}
e^{-C n \left(\frac{1}{e^{\varepsilon_0}} \wedge \frac{(e^{\varepsilon}-1)^2}{ (e^{\varepsilon}+1)^2 (e^{\varepsilon_0} - e^{-\varepsilon_0})^2}\right)}
&=
e^{-\Omega\left(\frac{n}{e^{\varepsilon_0}} \left(1 \wedge \frac{\varepsilon^2}{(1 \wedge \varepsilon_0^2) e^{\varepsilon_0}}\right)\right)} \notag \\
&=
e^{-\Omega\left(\frac{n \varepsilon^2}{(1 \wedge \varepsilon_0^2) e^{2\varepsilon_0}}\right)} \enspace,
\label{eqn:hoeff_term1}
\end{align}
where the last step uses that $\varepsilon \leq \varepsilon_0$ implies $\varepsilon^2 \leq (1 \wedge \varepsilon_0^2) e^{\varepsilon_0}$.
A similar argument based on the same bounds also yields
\begin{align}\label{eqn:hoeff_term2}
\frac{(e^{\varepsilon} + 1)^2 (e^{\varepsilon_0} - e^{-\varepsilon_0})^2}{4 n (e^{\varepsilon}-1)}
=
O\left(\frac{(1 \wedge \varepsilon_0^2) e^{2\varepsilon_0}}{n \varepsilon}\right) \enspace.
\end{align}
Combining \eqref{eqn:hoeff_term1} and \eqref{eqn:hoeff_term2} we obtain that $\mech{M}$ is $(\varepsilon,\delta)$-DP as long as
\begin{align*}
O\left(\frac{(1 \wedge \varepsilon_0^2) e^{2\varepsilon_0}}{n \varepsilon}\right) \cdot
e^{-\Omega\left(\frac{n \varepsilon^2}{(1 \wedge \varepsilon_0^2) e^{2\varepsilon_0}}\right)} \leq \delta \enspace.
\end{align*}
Taking $\varepsilon = c (1 \wedge \varepsilon_0) e^{\varepsilon_0} \sqrt{\log(1/\delta) / n}$ for some constant $c >0$, this translates to
\begin{align*}
O\left(\frac{(1 \wedge \varepsilon_0) e^{\varepsilon_0}}{c \sqrt{n \log(1/\delta)}}\right)
\cdot
e^{-\Omega\left(c^2 \log(1/\delta)\right)} \leq \delta \enspace.
\end{align*}
The result now follows from the assumption $\varepsilon_0 \leq \log(n / \log(1/\delta)) / 2$ after making an appropriate choice for $c$.
\end{proof}

\subsection{Proofs from Section~\ref{sec:improved_bounds}}

\begin{proof}[Proof of Lemma~\ref{lem:Lbenn}]
Let $\rv{L} = \sum_{i=1}^m \rv{L}_i$.
Under our assumptions on $\rv{L}_i$ we can apply Bennett's inequality \cite[Theorem 2.9]{boucheron2013concentration} to show that
\begin{align*}
\Pr[\rv{L} > t]
=
\Pr[\rv{L} - \Ex\rv{L} > t + a m]
\leq
e^{- \frac{m c}{b_+^2} \phi\left(\frac{(t + am) b_+}{m c}\right)} \enspace.
\end{align*}
Following the same argument used to prove Lemma~\ref{lem:Lhoeff} we get
\begin{align*}
\Ex[\rv{L}]_+
\leq
\int_0^{\infty} e^{- \frac{m c}{b_+^2} \phi\left(\frac{(t + am) b_+}{m c}\right)}
\leq
\frac{b_+}{a m \log\left(1 + \frac{a b_+}{c}\right)} e^{- \frac{m c}{b_+^2} \phi\left(\frac{a b_+}{c}\right)}
\enspace,
\end{align*}
where we used $\phi'(u) = \log(1 + u)$.
\end{proof}

\begin{proof}[Proof of Lemma~\ref{lem:Labc_RR}]
Note that for an $\varepsilon_0$-LDP randomized response mechanism $\mech{R} : [k] \to [k]$ we have a uniform blanket distribution $\omega(y) = 1/k$ and $\nu_x(y) = \one[y = x]$.
Thus, we obtain (\ref{it:range_l_rr}) by noting that for any $x, x', y \in [k]$ we have
\begin{align*}
\frac{\mu_x(y) - e^{\varepsilon} \mu_{x'}(y)}{\omega(y)}
&=
\gamma (1 - e^{\varepsilon}) + (1 - \gamma) k (\one[y = x] - e^{\varepsilon} \one[y = x']) \\
&\in
[\gamma (1 - e^{\varepsilon}) - (1 - \gamma) k e^{\varepsilon}, \gamma (1 - e^{\varepsilon}) + (1 - \gamma) k ] \enspace.
\end{align*}

To obtain (\ref{it:var_l_rr}) we first expand the definition of $\rv{L}$ to see that
\begin{align*}
\Ex\rv{L}^2
=
\Ex\left[\left(\gamma (1-e^{\varepsilon}) + (1-\gamma) k (\one[\rv{W} = x] - e^{\varepsilon} \one[\rv{W} = x'])\right)^2\right] \enspace.
\end{align*}
Since for $x \neq x'$ we have $\Pr[\rv{W} = x] = \Pr[\rv{W} = x'] = 1/k$ and $\Pr[\rv{W} = x, \rv{W} = x'] = 0$, we can expand the square in the above expression to get
\begin{align*}
\Ex\rv{L}^2
&=
\gamma^2 (1 - e^{\varepsilon})^2 + (1-\gamma)^2 k (e^{2 \varepsilon} + 1) + 2 \gamma (1-\gamma) (1 - e^{\varepsilon})^2 \\
&=
\gamma (2 - \gamma) (1 - e^{\varepsilon})^2 + (1-\gamma)^2 k (1 + e^{2 \varepsilon}) \enspace.
\end{align*}
\end{proof}

\begin{proof}[Proof of Lemma~\ref{lem:Labc_Laplace}]
Recall from the proof of Lemma~\ref{lem:gamma_specific} that the blanket distribution of an $\varepsilon_0$-LDP Laplace mechanism on $[0,1]$ is given by the Laplace distribution $\omega(y) = \frac{\varepsilon_0}{2} e^{-\varepsilon_0 |y - 1/2|}$.
Therefore, for any $x \in [0,1]$ and $y \in \R$ we have
\begin{align*}
e^{-\varepsilon_0/2} \leq \frac{\mu_x(y)}{\omega(y)} \leq e^{\varepsilon_0/2} \enspace,
\end{align*}
which implies (\ref{it:range_l_lap}) since for any $x, x' \in [0,1]$ and $y \in \R$:
\begin{align*}
\frac{\mu_x(y) - e^{\varepsilon} \mu_{x'}(y)}{\omega(y)}
\in
[e^{-\varepsilon_0/2} (1 - e^{\varepsilon+\varepsilon_0}), e^{\varepsilon_0/2} (1 - e^{\varepsilon-\varepsilon_0})] \enspace.
\end{align*}

To compute the second moment of $\rv{L}$ we proceed like in the proof of Lemma~\ref{lem:gamma_specific} and show that
\begin{align*}
\Ex\left[\left(\frac{\mu_{x}(W)}{\omega(W)}\right)^2\right]
&=
\int \frac{\mu_{x}(y)}{\omega(y)} \mu_{x}(y) dy
\\
&=
\frac{\varepsilon_0}{2} \int_{-\infty}^{\infty} e^{-2 \varepsilon_0 |y - x| + \varepsilon |y - 1/2|}
\\
&\leq
\frac{1}{3} (2 e^{\varepsilon_0/2} + e^{-\varepsilon_0})
\enspace,
\end{align*}
which is attained for $x = 0$ and $x=1$.
Furthermore, we have
\begin{align*}
\Ex\left[\frac{\mu_{x}(W) \mu_{x'}(W)}{\omega(W)^2}\right]
&=
\int \frac{\mu_{x}(y) \mu_{x'}(y)}{\omega(y)^2} \omega(y) dy
\\
&=
\frac{\varepsilon_0}{2} \int_{-\infty}^{\infty} e^{-\varepsilon_0 |y - x| - \varepsilon_0 |y - x'| + \varepsilon_0 |y - 1/2|}
\\
&\leq
2 e^{-\varepsilon_0/2} - e^{-\varepsilon_0}
\enspace,
\end{align*}
which is attained on $x= 0$ and $x' = 1$.
Putting these two bounds together we get
\begin{align*}
\Ex\rv{L}^2 \leq \frac{e^{2\varepsilon} + 1}{3} (2 e^{\varepsilon_0/2} + e^{-\varepsilon_0})
- 2 e^{\varepsilon} (2 e^{-\varepsilon_0/2} - e^{-\varepsilon_0}) \enspace.
\end{align*}
\end{proof}

\end{document}